%% file: iclr2027_conference.tex
\documentclass{article} 
\usepackage{iclr2027_conference,times}
\usepackage{graphicx} 
\usepackage{booktabs}  
\usepackage{titletoc}
\usepackage{makecell}  
\usepackage{tabularx} 
\usepackage{amsthm} 
\usepackage{xcolor,colortbl}
\definecolor{calibrationblue}{HTML}{EAF3FB}
\usepackage{wrapfig}
\newcommand{\meanstd}[2]{%
    $#1$\,{\scriptsize $\pm #2$}}
\newcommand{\bestmeanstd}[2]{%
    $\mathbf{#1}$\,{\scriptsize $\boldmath\mathbf{\pm #2}$}} 

\input{math_commands.tex}

\newtheorem*{cateadjustmentrestated}
{Proposition~\ref{prop:exact_cate_adjustment} (restated)}
\newtheorem{assumption}{Assumption} 
\newtheorem{proposition}{Proposition}
\usepackage{hyperref}
\usepackage{url}
\usepackage{xcolor}

\title{CAFE: Counterfactual Prediction via Fast Posterior Estimation}

\author{
\textbf{Xinyan Han$^{1,*}$ \quad
Xiaoyu Lin$^{1,*}$ \quad
Hao Zou$^{2}$} \quad
\textbf{Xingxuan Zhang$^{1,2}$ \quad
Bo Li$^{1}$ \quad
Peng Cui$^{1,\dagger}$}\\
$^{1}$Tsinghua University \quad
$^{2}$Stable AI \\
\texttt{han-xy25@mails.tsinghua.edu.cn} \quad
\texttt{clementinexiaoyu@gmail.com} \\
\texttt{ahio@163.com} \quad
\texttt{xingxuanzhang@hotmail.com} \\
\texttt{libo@sem.tsinghua.edu.cn} \quad \texttt{cuip@tsinghua.edu.cn}
}

\iclrfinalcopy 
\begin{document}

\maketitle
\fancyhead{}
\lhead{Preprint}
\renewcommand{\headrulewidth}{0pt}
\begin{abstract} 
Counterfactual prediction estimates an individual's outcome under an alternative intervention given their factual observations. Such outcomes are generally not identifiable from observational data without additional assumptions. Even within the class of fully observed additive noise models (ANMs), different causal graphs can generate the same observational distribution yet imply different individual counterfactual outcomes. Predictions based on a single estimated graph ignore this structural uncertainty. We therefore target a Bayesian counterfactual posterior predictive distribution that combines predictions from plausible SCMs. 
 We introduce CAFE (\textbf{C}ounterf\textbf{A}ctual Prediction via \textbf{F}ast Posterior \textbf{E}stimation), an amortized inference framework that directly approximates the Bayesian counterfactual posterior predictive distribution. We pretrain a transformer-based model on synthetic counterfactual tasks generated from a diverse prior over ANMs. Given an observational dataset, an individual's factual observations, and an intervention, CAFE approximates the corresponding posterior predictive distribution in a single forward pass. Experiments show that CAFE accurately predicts individual counterfactual outcomes in identifiable settings and approximates the posterior predictive distribution when structural uncertainty induced by observationally indistinguishable causal graphs exists. Strong performance in realistic manufacturing and viticulture settings further demonstrates its empirical robustness beyond the assumptions of the training prior.
\end{abstract}

\input{data/1intro}

\input{data/2problem}
\input{data/3method}
\input{data/4experiments} 
\vspace{-2pt}
 \section{Conclusion} 
 \vspace{-2pt}
We developed CAFE for individual counterfactual prediction when the causal graph is unknown and observational data may support conflicting counterfactual outcomes. Our analysis characterizes counterfactual identifiability and structural uncertainty under fully observed ANMs. Building on this foundation, CAFE amortizes Bayesian counterfactual inference through in-context learning on synthetic tasks drawn from a diverse SCM prior. Experiments demonstrate accurate predictions in identifiable settings, preservation of uncertainty from observationally equivalent models, and empirical robustness beyond the training assumptions. Our theoretical analysis is based on fully observed ANMs with independent exogenous noises. Although Section 4.3 provides empirical evidence of robustness to violations of these assumptions, extending the theoretical characterization of counterfactual identifiability and posterior uncertainty to more general classes of SCMs remains an important direction for future work.

\subsection*{AI use statement}
We used generative AI tools to assist in the writing of proofs, assist with translation. We have not used generative AI tools to generate synthetic data sets, help develop theoretical models or conceptual frameworks, formulate mathematical claims, provide critical ingredients for proving mathematical claims, propose or refine hypotheses, design or provide feedback on research methodology or experiments, implement methods, clean and reformat dataset, support qualitative and thematic data analysis, interpret results.  Additionally, we used generative AI tools for summarizing or analyse existing literature, sourcing/searching for information and identify relevant literature.
We take responsibility for the final content of this work, including text, claims or artifacts produced with the aid of generative AI.

\bibliography{iclr2027_conference}
\bibliographystyle{iclr2027_conference}

\clearpage
\appendix

\startcontents[appendices]
\section*{Appendix Contents}
\printcontents[appendices]{}{1}{\setcounter{tocdepth}{2}}

\clearpage
\input{data/appendix}

\end{document}

%% file: math_commands.tex
\usepackage{amsmath,amsfonts,bm}

\def\eqref#1{equation~\ref{#1}}

\def\1{\bm{1}}

\DeclareMathAlphabet{\mathsfit}{\encodingdefault}{\sfdefault}{m}{sl}
\SetMathAlphabet{\mathsfit}{bold}{\encodingdefault}{\sfdefault}{bx}{n}



%% file: data/1intro.tex
\section{Introduction} 
Counterfactual prediction asks what an individual's outcome would have been under an alternative intervention, given their factual observations~\citep{pearl2013structural}. The ability to reason about such alternative outcomes supports causal explanation~\citep{halpern2005causes}, the evaluation of past actions~\citep{oberst2019counterfactual}, and individualized decision-making~\citep{karimi2020algorithmic}. A fundamental challenge is that counterfactual outcomes are generally not identifiable from observational data without additional assumptions on the data-generating process~\citep{zhang2015estimation,spirtes2016causal,bodik2026retrospective}. Even when the observational distribution is known exactly, individual counterfactual non-identifiability can arise from uncertainty about the underlying causal graph, the structural equations, or the individual's exogenous noise values.

A common approach to make individual counterfactual outcomes identifiable is to assume that the correct causal graph or ordering is known and impose restrictions on the structural equations. Many SCM-based methods learn causal mechanisms conditional on a supplied or estimated graph or ordering~\citep{pawlowski2020dscm,sanchez2022vaca,khemakhem2021caurefl,javaloy2023causalNF,chao2023modeling,scetbon2024fixed,mahajan2024amortized}. In practice, the true structure is often unavailable~\citep{Maathuis_2009}, and causal discovery may not identify it uniquely~\citep{NEURIPS2024_f03fc354}. SCMs with different graphs can generate the same observational distribution yet imply different counterfactual outcomes.  
Predictions based on a single estimated graph ignore this structural uncertainty and can lead to poor predictions.  

To account for this structural uncertainty, we target a Bayesian counterfactual posterior predictive distribution. We remove the requirement of explicitly estimating causal graph and retain restrictions
on the structural equations by considering fully observed, acyclic additive noise models (ANMs) with independent exogenous noises. Given the exact observational distribution and an individual's complete factual observations, each compatible graph determines a unique counterfactual outcome under the conditions of Proposition~\ref{prop:counterfactual_identifiability}. Under these conditions, individual counterfactual non-identifiability arises only when compatible graphs imply different outcomes. The posterior predictive distribution combines these outcomes according to the graphs' posterior probabilities under a prior over ANMs. It is a point mass when the outcome is identifiable and remains non-degenerate when graphs with positive posterior probability imply different outcomes. 

However, exact computation of this posterior predictive distribution is generally intractable. The number of possible causal graphs grows super-exponentially with the number of variables. Computing their posterior probabilities also requires integrating over the structural mechanisms, and these integrals generally lack closed-form solutions for flexible nonlinear models.  To address these  challenges, we propose CAFE (\textbf{C}ounterf\textbf{A}ctual Prediction via \textbf{F}ast Posterior \textbf{E}stimation), an amortized inference framework that directly approximates the Bayesian counterfactual posterior predictive distribution. Following the prior-data fitted network approach~\citep{hollmann2025accurate,robertson2026pfn,balazadeh2026causalpfn}, we pretrain a transformer on synthetic counterfactual tasks generated from the ANM prior. Each task includes an observational dataset, an individual's factual observations, and a specified intervention as inputs, with the corresponding counterfactual outcome as the prediction target. Training amortizes Bayesian integration over causal structures and mechanisms, enabling CAFE to predict a counterfactual distribution from these inputs in a single forward pass without graph search or parameter updates. Our contributions are threefold: 
\begin{enumerate}
    \item We propose CAFE, an amortized inference framework for Bayesian counterfactual prediction that requires neither a known causal graph nor explicit graph estimation. CAFE is pretrained on synthetic tasks generated from a diverse prior over causal graphs, mechanisms, and treatment--outcome relationships, supporting counterfactual prediction across varied causal structures (Section~\ref{sec:prior_generation}).
    \item Under our modeling assumptions, CAFE accurately predicts individual counterfactual outcomes in identifiable settings (Section~\ref{sec:exp_csuite}) and approximates the posterior predictive distribution when observationally indistinguishable causal structures imply different outcomes, preserving uncertainty that observational data alone cannot resolve  (Section~\ref{sec:exp_uncertainty}).
    
    \item We evaluate CAFE on realistic manufacturing and viticulture settings. Strong predictive performance on these tasks demonstrates empirical robustness beyond the assumptions underlying our training prior (Section~\ref{sec:realistic}).
\end{enumerate}

%% file: data/2problem.tex
\section{Problem Setup} 
\label{sec:problem_setup}
\paragraph{Notations.}
A structural causal model (SCM) $\mathcal M=(G,\theta)$ describes
the data-generating process of endogenous variables
$\mathbf V=(V_1,\ldots,V_d)$ through the structural assignments
\begin{equation}
V_i = F_i\!\left(\mathbf V_{\operatorname{pa}G(i)},U_i;\theta\right),
\qquad i=1,\ldots,d.
\end{equation}
where $\mathbf U=(U_1,\ldots,U_d)$ denotes the exogenous variables
and $\operatorname{pa}_G(i)$ the parent indices of node $i$
in the causal graph $G$, and $\mathbf V_{\operatorname{pa}_G(i)}$
is the corresponding subvector.
The parameters $\theta$ specify both the structural functions
and the joint distribution of the exogenous
variables $\mathbf U=(U_1,\ldots,U_d)$.
Let $P_{\mathbf V}$ denote the induced observational distribution
and $\mathcal D_{\mathrm{obs}}=\{\mathbf v^{(j)}\}_{j=1}^{n}$
a dataset of $n$ i.i.d.\ observations drawn from $P_{\mathbf V}$.

\paragraph{Counterfactual prediction.}
Counterfactual prediction asks what an individual's outcome
would have been under a different treatment. For an individual with factual observations $\mathbf V=\mathbf v$,
let $T=V_j$ be the intervention variable and $Y=V_k$ the outcome,
with $j\neq k$. We denote by $Y^t$ the outcome that this same individual would
have experienced under $\operatorname{do}(T=t)$. The counterfactual outcome $Y^t$ is defined as the value of $Y$ in the modified SCM obtained by removing all incoming edges to $T$ and replacing the structural equation for $T$ with $T=t$, while keeping all other mechanisms and the individual's exogenous realization unchanged~\citep{pearl2013structural}. Our goal is to predict $Y^t$ from the individual's factual observations $\mathbf V=\mathbf v$ and the observational dataset
$\mathcal D_{\mathrm{obs}}$, without access to the underlying SCM or its causal graph.
Individual counterfactual outcomes are generally not identifiable from observational data without additional constraints on the data-generating process~\citep{bodik2026retrospective}. Different SCMs can induce the same observational distribution yet imply different values of $Y^t$ for an individual with factual observations $\mathbf V=\mathbf v$. To constrain the possible data-generating processes, we consider fully observed additive noise models (ANMs), as formalized below. Notice that this assumption restricts the class of admissible SCMs but does not by itself guarantee counterfactual identifiability.

\begin{assumption}[Fully observed additive noise models (ANMs)]
\label{ass:additive_noise}
The causal graph $G$ is acyclic, and all endogenous variables
$\mathbf V$ are observed. For each node $i$, the structural function satisfies
\begin{equation}
    F_i\!\left(\mathbf V_{\operatorname{pa}_G(i)},U_i; \theta \right)
    =
    f_i\!\left(\mathbf V_{\operatorname{pa}_G(i)};\theta \right)+U_i,
    \qquad i=1,\ldots,d,
\end{equation}
where $f_i$ is a deterministic function and $U_1,\ldots,U_d$ are mutually independent.
\end{assumption}
The following discussion assumes the regularity and support
conditions specified in Appendix~\ref{app:counterfactual_identifiability},
which also contains the formal statement and proof.

\begin{proposition}[Individual counterfactual identifiability, informal]
\label{prop:counterfactual_identifiability}
Within the class of SCMs satisfying
Assumption~\ref{ass:additive_noise},
the observational distribution $P_{\mathbf V}$ and a compatible
causal graph $G$ uniquely determine the counterfactual outcome
$y_G^t$ for given factual observations $\mathbf v$ and intervention
$\operatorname{do}(T=t)$.
Consequently, if $G$ is identifiable from $P_{\mathbf V}$
within this model class, the individual counterfactual outcome
is identifiable.
\end{proposition} 
When $P_{\mathbf V}$ is known exactly,
Proposition~\ref{prop:counterfactual_identifiability}
implies that each compatible graph determines a unique
counterfactual outcome for a given individual and intervention.
Any remaining counterfactual uncertainty therefore arises
from disagreement between compatible graphs. 

Let $\Pi$ be a prior over the SCMs considered above, and let
$\{G_1,\ldots,G_K\}$ denote the graphs compatible with
$P_{\mathbf V}$ that have positive posterior probability.
Write $\mathcal Y(\mathbf v,t;G,\theta)$ for the counterfactual
outcome of $Y$ under $\operatorname{do}(T=t)$ for an individual
with factual observations $\mathbf v$ in the SCM $(G,\theta)$.
For each compatible graph $G$, let $\theta_G(P_{\mathbf V})$
denote the parameters of any SCM with graph $G$ that induces
$P_{\mathbf V}$.
The Bayesian counterfactual posterior is
\begin{equation}
p^\star_\Pi(\,\cdot\mid\mathbf v,t,P_{\mathbf V})
=
\sum_{i=1}^{K}
p_\Pi(G_i\mid P_{\mathbf V})\,
\delta_{\mathcal Y
(\mathbf v,t;G_i,\theta_{G_i}(P_{\mathbf V}))},
\label{eq:counterfactual_population_posterior}
\end{equation}
where $\delta_a$ denotes a point mass at $a$. The mixture weights are determined by both $P_{\mathbf V}$ and the specified SCM prior $\Pi$.
 When the graph is identifiable from $P_{\mathbf V}$,
$K=1$ and the posterior reduces to a single point mass.
When multiple graphs have positive posterior probability, i.e., $K \ge 2$, 
the posterior is a weighted mixture of their corresponding
point masses.
If these graphs imply different counterfactual outcomes,
the posterior preserves uncertainty that the observational
distribution alone cannot resolve.
Counterfactual prediction methods should therefore account
for this uncertainty rather than assume that a single
estimated graph is correct.

%% file: data/3method.tex
\section{Method}
\vspace{-2pt}
\begin{figure}[t]
    \centering
    \includegraphics[width=0.9\linewidth]{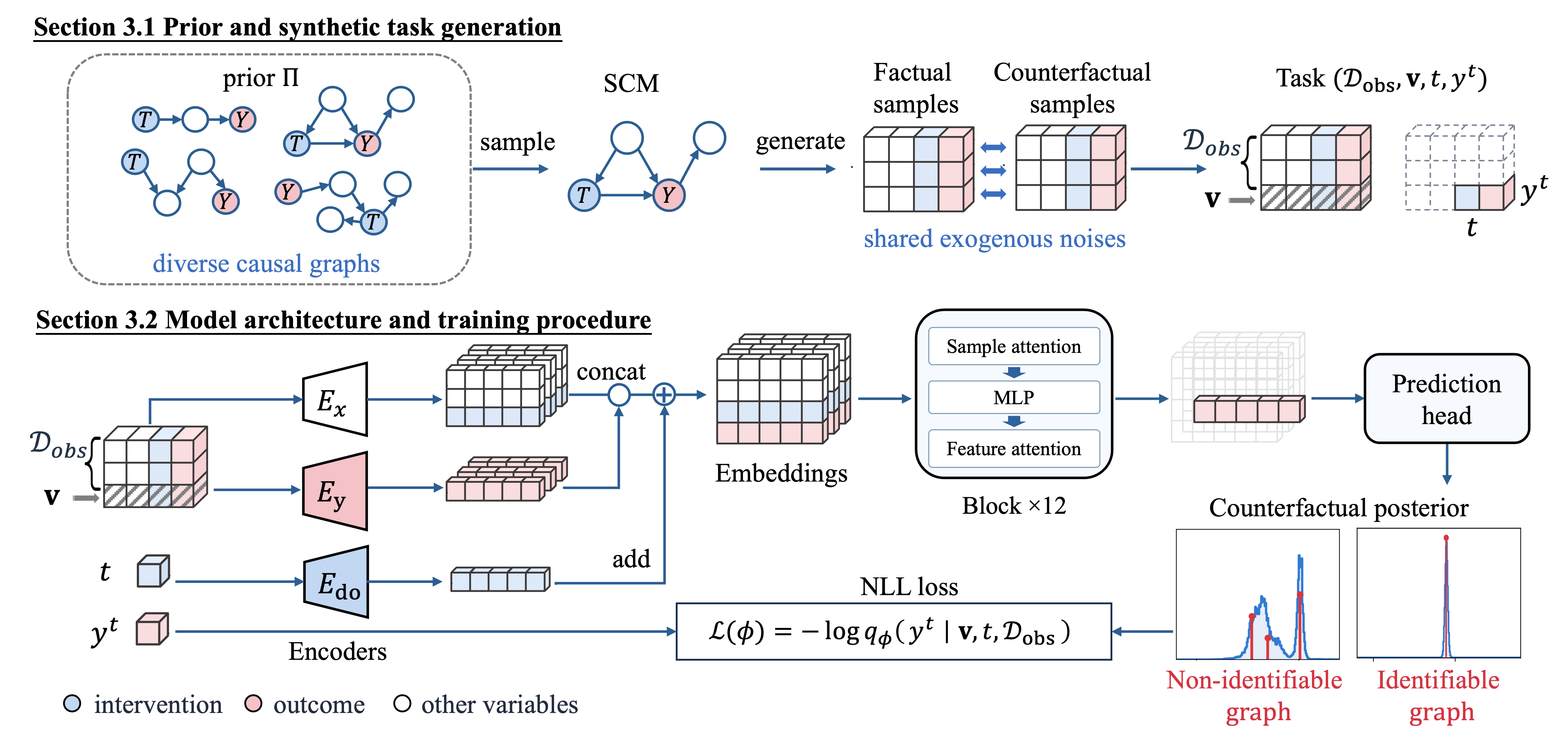}
    \caption{Overview of CAFE.
   \small{Top: We sample SCMs with diverse DAG structures from the prior $\Pi$ and generate paired factual and counterfactual data with shared exogenous noise. Bottom: We train a Transformer on these tasks to approximate the counterfactual posterior predictive distribution from observational data $\mathcal D_{\mathrm{obs}}$, an individual's factual observations $\mathbf v$, and an intervention $t$.}}
    \label{fig:framewrok} 
    \vspace{-3pt}
\end{figure} 
Section~\ref{sec:problem_setup} shows that disagreement between
compatible graphs can leave counterfactual outcomes uncertain
even when $P_{\mathbf V}$ is known exactly.
In practice, we have access only to a finite observational
dataset $\mathcal D_{\mathrm{obs}}$, which introduces additional
uncertainty about both the graph $G$ and the SCM parameters
$\theta$.
The corresponding Bayesian counterfactual posterior predictive
distribution is
\begin{equation}
p^\star_\Pi(\,\cdot\mid\mathbf v,t,\mathcal D_{\mathrm{obs}})
=
\sum_{G\in\mathcal G_\Pi}
p_\Pi(G\mid \mathbf v, \mathcal D_{\mathrm{obs}})
\int_{\Theta_G}
\delta_{\mathcal Y(\mathbf v,t;G,\theta)}\,
p_\Pi(d\theta\mid \mathbf v, G,\mathcal D_{\mathrm{obs}}).
\label{eq:post_finite}
\end{equation}
Here, $\mathcal G_\Pi$ denotes the graphs supported by the prior $\Pi$, and $\Theta_G$ is the
SCM parameter space for graph $G$.
Under Assumption~\ref{ass:additive_noise}, the factual observations
$\mathbf v$ uniquely determine the individual's noise realization
for each fixed SCM $(G,\theta)$.
Keeping this realization unchanged under intervention yields
the deterministic outcome $\mathcal Y(\mathbf v,t;G,\theta)$.
Eq.~\ref{eq:post_finite} averages the corresponding point
masses over the posterior uncertainty in $G$ and $\theta$. For each fixed intervention value $t$, under the prior $\Pi$, this posterior converges weakly almost surely
to the population posterior in Eq.~\ref{eq:counterfactual_population_posterior}
as the observational sample size tends to infinity. The conditions and proof are provided in Appendix~\ref{app:counterfactual_posterior_limit}.
\vspace{-2pt}
Exact evaluation of Eq.~\ref{eq:post_finite} is generally intractable. In particular, the posterior distributions over the causal graph $G$ and mechanism parameters $\theta$ are typically not available in closed form, and computing the posterior predictive distribution requires marginalization over a combinatorial space of graph structures and a potentially high-dimensional space of mechanism parameters.
As illustrated in Figure~\ref{fig:framewrok}, we amortize this computation by training a counterfactual posterior predictor $q_\phi(y^t\mid\mathbf v,t,\mathcal D_{\mathrm{obs}})$ on synthetic counterfactual tasks.
Each task $(\mathcal D_{\mathrm{obs}},\mathbf v,t,y^t)$ is generated from an SCM sampled from $\Pi$.
The predictor receives the observational dataset $\mathcal D_{\mathrm{obs}}$, factual state $\mathbf v$, intervention $t$ and learns to predict the counterfactual outcome $y^t$. Let $P_\Pi$ denote the joint distribution induced by the prior.
We train $q_\phi$ to learn the counterfactual posterior predictive distribution by minimizing the expected negative log-likelihood:
\begin{equation}
\mathcal L(\phi)
=
\mathbb E_{(\mathcal D_{\mathrm{obs}},\mathbf v,t,y^t)\sim P_\Pi}
\left[
-\log q_\phi(y^t\mid\mathbf v,t,\mathcal D_{\mathrm{obs}})
\right].
\label{eq:counterfactual_training_objective}
\end{equation}
Thus, with sufficient model capacity and exact optimization of the population objective, the optimal predictor matches the Bayesian posterior predictive distribution for almost every input under $P_\Pi$.
At inference time, the trained model takes an observational dataset, an individual's factual state, and a specified intervention, and approximates this distribution in a single forward pass.
\vspace{-2pt}
\subsection{Prior and Synthetic Task Generation}
\label{sec:prior_generation}
\vspace{-2pt}
We instantiate prior $\Pi$ through a sampling procedure over causal graphs, structural functions, and exogenous distributions.
Below, we describe the main components of the prior and the task-generation procedure. Sampling hyperparameters are provided in Appendix~\ref{app:prior_data_generation}.
\vspace{-2pt}
\paragraph{Causal graphs $G$.}
We first sample the number of nodes $d$ and a target average in-degree $\bar{k}$. We then sample a random node ordering and select $\min\{\lfloor d\bar{k}\rfloor, d(d-1)/2\}$ edges uniformly without replacement from those pointing forward in this ordering, ensuring the graph is acyclic.
To cover diverse causal relationships, we randomly select two distinct nodes as the intervention variable $T=V_j$ and the outcome variable $Y=V_k$. Unlike the back-door priors used by~\cite{balazadeh2026causalpfn}, our prior accommodates a broader range of practical settings where the remaining observed variables include treatment descendants, rather than only pre-treatment covariates. For example, when predicting coronary events ($Y$) under simvastatin treatment ($T$), a patient's factual observations may include LDL cholesterol levels measured during treatment ($X_1$), which are themselves affected by the medication~\citep{4S}. A different configuration arises when predicting progression to severe acute kidney injury ($Y$) under an alternative furosemide dose ($T$)~\citep{chawla2013development}. Urine output following treatment ($X_1$) depends on both the dose and the extent of renal tubular injury ($X_2$), which also influences subsequent disease progression. These relationships form the path $T\to X_1\leftarrow X_2\to Y$. By including such causal structures, our prior accounts for the diverse causal roles of observed variables. An ablation study in Appendix~\ref{app:albation} shows that restricting pretraining to back-door structures can reduce prediction accuracy on other causal structures.
\vspace{-3pt}
\paragraph{Structural functions and exogenous distributions $\theta$.}
For each SCM, we assign each root node $V_i=U_i$ a distribution chosen from the Gaussian, uniform, and Zipf families, then sample root-node values independently from their assigned distributions. For each non-root node $i$, we instantiate $f_i$ as a randomly initialized MLP with $|\operatorname{pa}_G(i)|$ inputs and one scalar output. We then sample a signal-to-noise ratio (SNR) to vary the strength of the additive noise relative to the structural signal.
The noise family is selected from Gaussian, Laplace, Gumbel, uniform, and exponential distributions; we also include a noiseless option with $U_i=0$.
These choices allow the prior to represent both symmetric and asymmetric disturbances, as well as deterministic mechanisms.
\vspace{-3pt}
\paragraph{Factual and counterfactual generation.}
Given the sampled SCM $\mathcal M=(G,\theta)$, we jointly generate factual and counterfactual samples in topological order. At each non-intervened node, we apply the same structural function to the factual and counterfactual parent values, sharing one noise realization within each sample pair: 
\vspace{-2pt}
\begin{equation}
v_i = f_i\left(\mathbf v_{\operatorname{pa}_G(i)}\right)+u_i,\qquad
v_i^t = f_i\left(\mathbf v^t_{\operatorname{pa}_G(i)}\right)+u_i,
\qquad i\neq j,
\label{eq:prior_counterfactual_generation}
\end{equation}
where $f_i \equiv 0$ for root nodes.
Before selecting the intervention value, we randomly assign individuals
to an observational group and a target group.
Upon reaching the treatment node $T=V_j$, we generate its factual
values for all individuals and sample an intervention value $t$
uniformly from the factual treatment values in the observational group.
We then set intervention value $v_j^t=t$ for all individuals and continue the traversal, propagating the intervention to generate the counterfactual values.
The factual observations of the observational group form
$\mathcal D_{\mathrm{obs}}$.
Each target individual's factual state $\mathbf v$ and paired
counterfactual outcome $y^t=v_k^t$ yield a training task
$(\mathcal D_{\mathrm{obs}},\mathbf v,t,y^t)$.
\vspace{-2pt}
\subsection{Model architecture and training procedure} 
\vspace{-3pt}
We implement $q_\phi$ using a tabular Transformer as illustrated in Figure~\ref{fig:framewrok} adapted for counterfactual prediction. The model takes an observational dataset $\mathcal D_{\mathrm{obs}}$, an individual's complete factual observations $\mathbf v$ (including the observed outcome $y$), and an intervention value $t$. It predicts the distribution of that same individual's outcome $Y^t$ under $\operatorname{do}(T=t)$. 
\vspace{-3pt}
\paragraph{Input encoding.}
Let $d$ denote the number of observed variables, including the treatment $T=V_j$ and outcome $Y=V_k$. We form an $(N+M)\times d$ input table, where the first $N$ rows contain observations from $\mathcal D_{\mathrm{obs}}$ and the remaining $M$ rows contain the factual observations of the individuals whose counterfactual outcomes we predict. Three separate two-layer MLP encoders map scalar inputs to 192-dimensional embeddings. Both $E_x$ and $E_y$ are applied to all $N+M$ rows: $E_x$ is shared across all factual variables other than $Y$, while $E_y$ encodes factual outcomes. We apply $E_{\mathrm{do}}$ only to the intervention values specified for the last $M$ rows and add the resulting embeddings to the corresponding factual treatment embeddings. This associates the intervention with its target variable while retaining the individual's observed treatment and outcome, yielding an embedding tensor of shape $(N+M)\times d\times192$. 
\vspace{-3pt}
\paragraph{Transformer backbone.}
We use the 12-layer Transformer architecture of LimiX~\citep{zhang2025limix}. Each block applies sample attention, a position-wise MLP, and feature attention in sequence (the SMF configuration), with residual connections and layer normalization. Sample attention gathers information from the $N$ observational rows separately for each variable. Feature attention combines information across variables within the same row. Together, these operations combine each individual's factual observations and intervention with information from $\mathcal D_{\mathrm{obs}}$.  
\vspace{-3pt}
\paragraph{Training procedure.}
We jointly train the encoders, Transformer backbone, and an MLP prediction head on synthetic tasks generated from $\Pi$ using the procedure described in Section~\ref{sec:prior_generation}. We discretize the continuous counterfactual outcome space into $B=5000$ bins using the boundaries provided by TabPFN2.5~\citep{grinsztajn2025tabpfn}. For each individual, the prediction head outputs a probability for each bin, representing the probability that their counterfactual outcome falls within that interval. We minimize the negative log-probability assigned to the bin containing the counterfactual outcome, averaged over the $M$ individuals. This is the discretized version of the objective in Eq.~\ref{eq:counterfactual_training_objective}. Training hyperparameters are provided in Appendix~\ref{app:hyperparameters}.

%% file: data/4experiments.tex
\section{Experiments} 
\vspace{-2pt}
We design our experiments to answer four questions: (1) Can CAFE accurately predict individual counterfactual outcomes when the causal graph is identifiable and the conditions of Proposition~\ref{prop:counterfactual_identifiability} hold (Section~\ref{sec:exp_csuite})? (2) Within the same class of Assumption~\ref{ass:additive_noise}, can it approximate the counterfactual posterior predictive distribution when observationally equivalent graphs imply different counterfactual outcomes, preserving the uncertainty that observational data alone cannot resolve (Section~\ref{sec:exp_uncertainty})? (3) How robust is its predictive performance in realistic settings where unobserved variables or non-additive noise may violate our modeling assumptions (Section~\ref{sec:realistic})? (4) Can methods for estimating treatment effects be used to predict individual counterfactual outcomes (Appendix~\ref{app:cepo})?
\vspace{-2pt}
\paragraph{Baselines} We compare CAFE against 7 methods with different approaches to causal modeling and inference. All methods are evaluated without access to the ground-truth graph or causal ordering. CAREFL~\citep{khemakhem2021caurefl} and CausalNF~\citep{javaloy2023causalNF} start from a partially directed graph estimated using PC: CAREFL orients the remaining edges using its likelihood-ratio criterion, while CausalNF models unresolved components as blocks with fixed internal orderings. FiP predicts a causal ordering using a pretrained ordering model. DECI~\citep{geffner2022DECI} and DiBS~\citep{lorch2021dibs} take a Bayesian approach to causal structure learning: DECI learns a variational distribution over graphs together with point estimates of additive-noise mechanisms, while DiBS jointly approximates a posterior over graphs and mechanism parameters. We also evaluate VCI~\citep{wu2025counterfactual} and SCIGAN~\citep{bica2020estimating}, which assume that observed covariates suffice to control treatment--outcome confounding, avoiding explicit graph estimation. Implementation details for all baselines are provided in Appendix~\ref{app:baselines}. 
\vspace{-2pt}
\paragraph{Evaluation protocol.}
Across all experiments, we randomly partition the individuals in each dataset into two equal folds. We use the factual observations from one fold as an observational dataset of size $N$ and evaluate counterfactual predictions for the $M=N$ individuals in the other, then swap their roles so that every individual is evaluated once. We repeat this procedure with five partition seeds and report the mean and standard deviation of each evaluation metric across seeds.
\subsection{Counterfactual Prediction under Identifiability}
\label{sec:exp_csuite} 
\vspace{-2pt}
\paragraph{C-Suite}
We first examine whether CAFE can accurately predict identifiable individual counterfactual outcomes under our causal assumptions. We evaluate individual counterfactual prediction on C-Suite~\citep{geffner2022DECI}, a synthetic benchmark covering classical causal structures and inference scenarios. We consider additive-noise settings with at least three variables. For each dataset, the causal graph is uniquely determined by the exact observational distribution within the class of faithful independent ANMs with $C^3$ mechanisms. The remaining conditions of Proposition~\ref{prop:counterfactual_identifiability} also hold, making each individual's counterfactual outcome identifiable from the observational distribution, factual observations, and intervention (see Appendix~\ref{app:iden_csuite} for proofs). These are population-level guarantees; our experiments assess how accurately the outcomes can be predicted from finite observational data. Additional evaluations on C-Suite settings outside the additive-noise assumption are reported in Appendix~\ref{app:non_add_csuite}.  
\begin{table}[t]
    \centering
    \small
    \setlength{\tabcolsep}{3pt}
    \caption{Counterfactual prediction RMSE on 7 C-Suite datasets
    with additive noise (mean $\pm$ std).
    Bold indicates the lowest mean RMSE in each column, excluding the factual-outcome baseline.}
    \label{tab:csuite_results}
    \begin{tabular}{lccccccc}
        \toprule
        Method
        & \makecell{collider \\ lingauss}
        & \makecell{fork \\ lin nongauss}
        & \makecell{fork \\ nonlin gauss}
        & \makecell{large \\ backdoor}
        & \makecell{nonlin \\ simpson}
        & \makecell{symprod \\ simpson}
        & \makecell{weak \\ arrows} \\
        \midrule
        y factual
        & \meanstd{0.57}{0.00}
        & \meanstd{0.00}{0.00}
        & \meanstd{0.00}{0.00}
        & \meanstd{0.74}{0.00}
        & \meanstd{1.36}{0.00}
        & \meanstd{1.37}{0.00}
        & \meanstd{1.07}{0.00} \\
        \midrule
        CAREFL
        & \meanstd{0.11}{0.15}
        & \meanstd{0.34}{0.00}
        & \meanstd{0.02}{0.00}
        & \meanstd{0.61}{0.10}
        & \meanstd{1.36}{0.00}
        & \meanstd{1.37}{0.00}
        & \meanstd{0.21}{0.00} \\
        CausalNF
        & \meanstd{0.08}{0.01}
        & \meanstd{0.55}{0.43}
        & \meanstd{0.09}{0.11}
        & \meanstd{1.57}{0.82}
        & \meanstd{2.48}{0.16}
        & \meanstd{1.37}{0.00}
        & \meanstd{1.86}{0.89} \\
        VCI
        & \meanstd{0.51}{0.00}
        & \meanstd{0.34}{0.00}
        & \meanstd{0.28}{0.01}
        & \meanstd{0.70}{0.01}
        & \meanstd{1.58}{0.12}
        & \meanstd{1.37}{0.01}
        & \meanstd{1.02}{0.01} \\
        SCIGAN
        & \meanstd{0.81}{0.01}
        & \meanstd{0.62}{0.04}
        & \meanstd{0.47}{0.00}
        & \meanstd{1.20}{0.06}
        & \meanstd{1.07}{0.09}
        & \meanstd{1.51}{0.05}
        & \meanstd{0.98}{0.03} \\
        DECI
        & \meanstd{0.39}{0.30}
        & \meanstd{0.60}{0.46}
        & \meanstd{0.03}{0.07}
        & \meanstd{0.42}{0.23}
        & \meanstd{3.04}{1.07}
        & \meanstd{1.45}{0.10}
        & \meanstd{0.42}{0.17} \\
        DiBS 
        & \meanstd{0.36}{0.08}
        & \meanstd{0.28}{0.07}
        & \meanstd{0.03}{0.03}
        & \meanstd{0.62}{0.14}
        & \meanstd{1.89}{0.24}
        & \meanstd{1.34}{0.10}
        & \meanstd{0.59}{0.14}\\
        Fip
        & \meanstd{0.40}{0.08}
        & \meanstd{0.89}{0.13}
        & \bestmeanstd{0.00}{0.00}
        & \meanstd{0.74}{0.00}
        & \meanstd{2.56}{0.57}
        & \meanstd{1.49}{0.34}
        & \meanstd{1.07}{0.00} \\
        CAFE
        & \bestmeanstd{0.05}{0.01}
        & \bestmeanstd{0.01}{0.00}
        & \meanstd{0.01}{0.00}
        & \bestmeanstd{0.15}{0.02}
        & \bestmeanstd{0.93}{0.01}
        & \bestmeanstd{1.20}{0.01}
        & \bestmeanstd{0.14}{0.00} \\
        \bottomrule
    \end{tabular}
\end{table}  
\vspace{-2pt}
\paragraph{Metric}
For each C-Suite dataset, we use $N=2000$. In these identifiable settings, individual counterfactual outcomes are uniquely determined, and the predictive distributions shown in Appendix~\ref{app:identifiable_distributions} concentrate sharply near these outcomes, supporting our focus on point-prediction accuracy. We use the posterior predictive mean as the point prediction and report counterfactual RMSE. We also include a factual-outcome baseline that predicts $\hat y^{t}=y$, providing a reference for whether a method improves upon retaining the observed outcome.
\vspace{-2pt}
\paragraph{Results} 
Table~\ref{tab:csuite_results} shows that CAFE achieves the lowest reported RMSE among seven baselines on six of the seven datasets, demonstrating strong individual counterfactual prediction across diverse causal structures. In the fork settings, e.g., $T\leftarrow X\rightarrow Y$, intervening on $T$ leaves the individual's outcome unchanged. CAFE yields an RMSE of approximately $0.01$ in both settings, closely preserving the factual outcome. CAFE also exhibits low variation across seeds on all seven datasets, whereas CAREFL, CausalNF, DECI, and FiP vary substantially in several settings. Although the causal graphs are identifiable from the exact observational distributions under the stated assumptions, this does not guarantee accurate recovery of their structures or mechanisms by a learning algorithm given finite data~\citep{rahmadi2017causality,faller2024self}. Errors in estimated edges, causal orderings, or mechanisms can affect counterfactual predictions. 

\subsection{Counterfactual uncertainty under non-identifiable causal structure.} 
\label{sec:exp_uncertainty} 
\begin{figure}[t]
    \centering
    \includegraphics[width=0.85\linewidth]{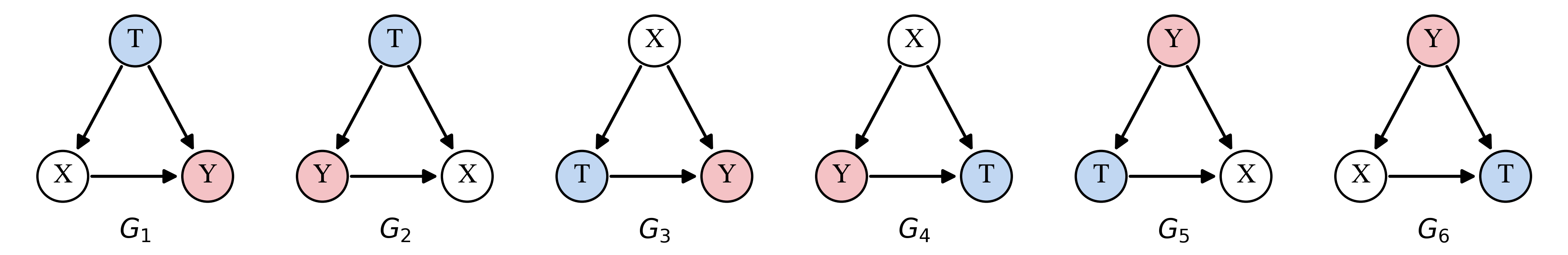}
    \par\vspace{0.0cm}
    \includegraphics[width=0.98\linewidth]{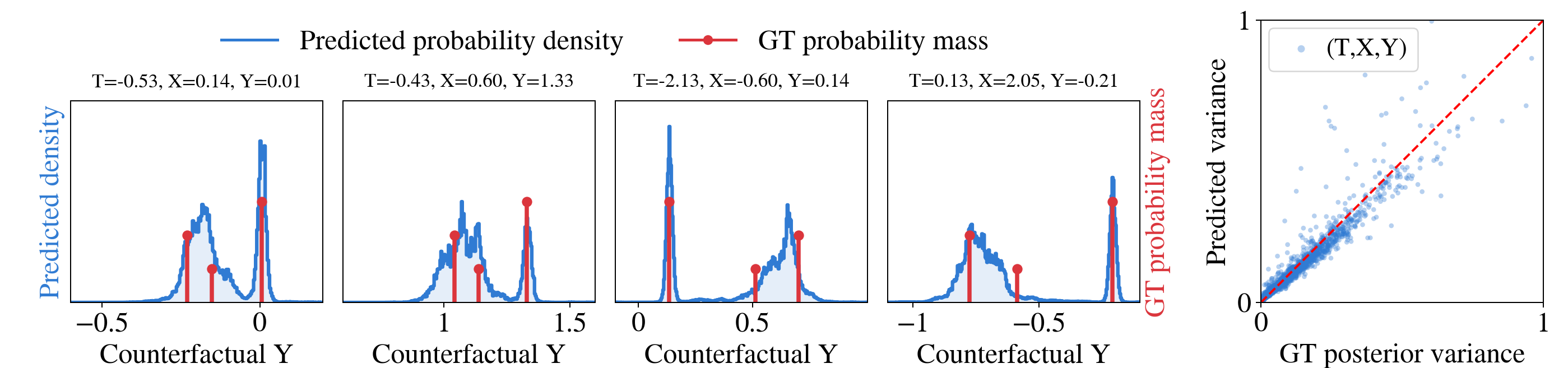}
    \caption{Counterfactual uncertainty under non-identifiable causal structure. \small{Top: six observationally indistinguishable SCMs. Bottom: predicted versus ground-truth counterfactual distributions for randomly selected individuals (left) and posterior variances across individuals (right) at $N=1000$.}}
    \label{fig:exp2_uncertainty}
\end{figure} 
\vspace{-2pt}
\paragraph{Observationally equivalent SCMs.}
We next evaluate whether CAFE can capture counterfactual uncertainty due to causal graph non-identifiability from observational data. We construct six fully observed linear Gaussian SCMs satisfying Assumption~\ref{ass:additive_noise} that induce the same observational distribution due to Markov equivalent causal structures. This construction provides an analytically tractable ground-truth posterior for evaluating uncertainty arising from observationally indistinguishable causal structures. Consider three observed variables $\mathbf V=(T,X,Y)$ with joint distribution
\begin{equation}
\mathbf V\sim\mathcal N(\mathbf 0,\Sigma),
\qquad
\Sigma=
\begin{pmatrix}
1 & 1/2 & 1/2\\
1/2 & 1 & 1/2\\
1/2 & 1/2 & 1
\end{pmatrix}.
\end{equation}
For each permutation $\pi$ of the three variables, we define an SCM with mutually independent exogenous noises:
\begin{equation}
\begin{aligned}
V_{\pi_1} &= U_{\pi_1},
& U_{\pi_1} &\sim \mathcal N(0,1),\\
V_{\pi_2} &= \tfrac12 V_{\pi_1}+U_{\pi_2},
& U_{\pi_2} &\sim \mathcal N(0,\tfrac34),\\
V_{\pi_3} &= \tfrac13 V_{\pi_1}+\tfrac13 V_{\pi_2}+U_{\pi_3},
& U_{\pi_3} &\sim \mathcal N(0,\tfrac23).
\end{aligned}
\end{equation}
These six complete DAGs are the only DAGs compatible
with $P_{\mathbf V}$ under
Assumption~\ref{ass:additive_noise}.
The six SCMs differ only in the ordering of $T$, $X$, and $Y$;
their structural coefficients and noise distributions follow
the same specification in each ordering.
Our prior samples these orderings uniformly and applies the
same rules for generating mechanisms and noise distributions
in each ordering.
Permuting the variables leaves $P_{\mathbf V}$ unchanged,
so conditioning on $P_{\mathbf V}$ preserves this symmetry.
The six compatible DAGs therefore receive equal posterior weights:
$p_\Pi(G_i\mid P_{\mathbf V})=1/6$ for $i=1,\ldots,6$.

For an individual with factual observations
$\mathbf v=(t_{\mathrm f},x,y)$, where $t_{\mathrm f}$ is
the observed treatment value, we compute the counterfactual
outcome under $\operatorname{do}(T=t)$ in each SCM while
keeping the individual's exogenous noise unchanged.
Substituting these outcomes and their weights into
Eq.~\ref{eq:counterfactual_population_posterior} gives
\begin{equation}
p^\star_\Pi(\,\cdot\mid\mathbf v,t,P_{\mathbf V})
=
\underbrace{\tfrac12\delta_y}_{G_4,\,G_5,\,G_6}
+
\underbrace{\tfrac13\delta_{y+\frac12(t-t_{\mathrm f})}}_{G_1,\,G_2}
+
\underbrace{\tfrac16\delta_{y+\frac13(t-t_{\mathrm f})}}_{G_3}.
\label{eq:ground_truth_counterfactual_posterior}
\end{equation}
We evaluate how closely CAFE's predictions from finite
observational dataset approximate this analytical population
posterior as the observational dataset size increases.
\vspace{-2pt}
\paragraph{Metrics} For each observational dataset size $N\in\{50,100,250,500,1000\}$, we set $t=-1$.
We assess first- and second-order accuracy using posterior mean RMSE and posterior variance RMSE, respectively. 
Beyond these moment summaries, we use the integrated quadratic distance (IQD)~\citep{Divergence} to measure the discrepancy between the model's continuous predictive distribution and the discrete ground-truth posterior: \begin{equation} \operatorname{IQD} = \frac{1}{2N}\sum_{i=1}^{2N} \int_{\mathbb R} \left[\widehat F_i(z)-F_i^{\mathrm{GT}}(z)\right]^2\,dz, \end{equation} where $\widehat F_i$ and $F_i^{\mathrm{GT}}$ are the predicted and ground-truth counterfactual CDFs for individual $i$. 
\vspace{-2pt}
\paragraph{Results}
\begin{wraptable}{r}{0.52\linewidth}
\vspace{-\baselineskip}
\centering
\small
\setlength{\tabcolsep}{4pt}
\caption{Counterfactual posterior prediction under non-identifiable
causal structure (mean $\pm$ std).}
\label{tab:posterior_prediction}
\resizebox{\linewidth}{!}{%
\begin{tabular}{rccc}
\toprule
$N$ & Mean RMSE & Variance RMSE & IQD \\
\midrule
50   & \meanstd{0.464}{0.120} & \meanstd{0.558}{0.132} & \meanstd{0.161}{0.050} \\
100  & \meanstd{0.342}{0.191} & \meanstd{0.565}{0.267} & \meanstd{0.104}{0.092} \\
250  & \meanstd{0.153}{0.048} & \meanstd{0.085}{0.017} & \meanstd{0.031}{0.017} \\
500  & \meanstd{0.110}{0.020} & \meanstd{0.073}{0.015} & \meanstd{0.020}{0.008} \\
1000 & \bestmeanstd{0.094}{0.012} & \bestmeanstd{0.055}{0.010} & \bestmeanstd{0.018}{0.004} \\
\bottomrule
\end{tabular}%
}
\end{wraptable}  
Table~\ref{tab:posterior_prediction} shows that CAFE's errors relative to the analytical population posterior generally decrease as the observational dataset grows. Additional observations help estimate the observational distribution more accurately, but cannot distinguish the six SCMs, which predict different counterfactual outcomes for the same individual. At $N=1000$, Figure~\ref{fig:exp2_uncertainty} (bottom left) shows that CAFE's predicted distributions closely match the analytical posteriors for the selected individuals. The variance comparison (bottom right) further supports the accuracy of CAFE's posterior variance estimates across individuals.
Additional comparisons with DECI and DiBS, two Bayesian methods that can model graph uncertainty, show that CAFE achieves the lowest errors on all three posterior metrics (Appendix~\ref{app:bayesian_posterior_comparison}). Together, these results support CAFE's ability to approximate the ground-truth counterfactual posterior and capture uncertainty due to graph non-identifiability.  
\subsection{Robustness to violations of modeling assumptions.}
\label{sec:realistic} 
\vspace{-2pt}
\paragraph{CausalMan.}
We evaluate predictive robustness on CausalMan~\citep{CausalMan}, a physics-based simulator modeled after a real production line that assembles magnetic valves into hydraulic units. Given a component's factual measurements, the task is to predict whether it would have passed quality inspection (\texttt{Sec\_C2\_Machine1\_ProcessResult}) under a different press-fitting force (\texttt{PF\_M1\_T1\_Force}). We use CausalMan-Micro, which simulates a single product and retains 24 observed variables after removing constants, and CausalMan-Small, which includes multiple products and 53 observed variables. Both have 104 unobserved variables. Both datasets contain mixed variable types, and their unobserved process variables and non-additive noise place them outside the fully observed ANMs assumption. Moreover, the observed variables include descendants of the intervention variable, so these tasks do not satisfy the common assumption that all covariates are unaffected by treatment.  
For each dataset, we use $N=1000$ observational samples and evaluate interventions at $t\in\{15000,30000\}$.
We threshold continuous counterfactual predictions at $0.5$ to obtain binary labels and report counterfactual prediction accuracy.
Table~\ref{tab:causalman_results} shows that CAFE achieves the highest accuracy in all four settings. Several baselines achieve accuracies close to those of the factual-outcome baseline.
CAREFL infers no directed path from the intervention variable to the outcome, while FiP places the outcome before the intervention variable in its inferred ordering. Both structural errors prevent the intervention from affecting the predicted outcome. On CausalMan-Small, DECI struggles to learn acyclic graphs; stronger acyclicity regularization yields overly sparse structures, with counterfactual accuracy matching the factual-outcome baseline. DiBS can recover correct causal relationships in both datasets, but inaccuracies in its learned structural equations limit its sensitivity to small deviations from the factual treatment value, as observed at $t=15000$. It is also computationally expensive, requiring approximately $48{,}000s$ on CausalMan-Small, compared with about $0.12s$ for CAFE (Appendix Table~\ref{tab:runtime_comparison}). VCI and SCIGAN also fail to improve counterfactual prediction over retaining the factual outcome. Both methods assume that covariates are unaffected by treatment, an assumption violated by the treatment descendants present in these datasets. Using the factual values of these descendants under intervention may explain why their predictions remain close to the factual outcome. These results demonstrate CAFE's robustness in realistic scenarios that violate the fully observed ANM assumptions underlying its training prior.

\begin{table}[t]
    \centering
    \small
    \setlength{\tabcolsep}{4pt}
    \caption{Counterfactual prediction performance on CausalMan and
    Sangiovese (mean $\pm$ std). We report accuracy
    (\%, higher is better) on CausalMan under two press-fitting force
    interventions, and MAPE (\%, lower is better) for Brix and pH
    on Sangiovese.}
    \label{tab:causalman_results}
    \begin{tabular}{lcccccc}
        \toprule
        Method
        & \multicolumn{2}{c}{CausalMan Micro}
        & \multicolumn{2}{c}{CausalMan Small}
        & \multicolumn{2}{c}{Sangiovese} \\
        & \multicolumn{2}{c}{Accuracy (\%) $\uparrow$}
        & \multicolumn{2}{c}{Accuracy (\%) $\uparrow$}
        & \multicolumn{2}{c}{MAPE (\%) $\downarrow$} \\
        \cmidrule(lr){2-3}
        \cmidrule(lr){4-5}
        \cmidrule(lr){6-7}
        & $t=15000$ & $t=30000$
        & $t=15000$ & $t=30000$
        & Brix & pH \\
        \midrule
        y factual
        & \meanstd{92.85}{0.00}
        & \meanstd{12.15}{0.00}
        & \meanstd{94.05}{0.00}
        & \meanstd{6.70}{0.00}
        & \meanstd{6.83}{0.00}
        & \meanstd{2.07}{0.00} \\
        \midrule
        CAREFL
        & \meanstd{92.85}{0.00}
        & \meanstd{12.15}{0.00}
        & \meanstd{94.05}{0.00}
        & \meanstd{6.70}{0.00}
        & \meanstd{2.52}{0.14}
        & \meanstd{1.88}{0.23} \\
        CausalNF
        & \meanstd{92.98}{0.26}
        & \meanstd{20.86}{17.42}
        & \meanstd{94.40}{0.54}
        & \meanstd{24.30}{21.61}
        & \meanstd{4.20}{0.25}
        & \meanstd{1.93}{0.17} \\
        VCI
        & \meanstd{92.79}{0.06}
        & \meanstd{12.09}{0.06}
        & \meanstd{94.04}{0.02}
        & \meanstd{6.69}{0.02}
        & \meanstd{5.70}{0.25}
        & \meanstd{1.83}{0.07} \\
        SCIGAN
        & \meanstd{91.48}{2.67}
        & \meanstd{12.15}{0.00}
        & \meanstd{92.80}{2.50}
        & \meanstd{6.68}{0.02}
        & \meanstd{6.58}{0.08}
        & \meanstd{2.76}{0.12} \\
        DECI
        & \meanstd{95.85}{0.79}
        & \meanstd{53.87}{14.58}
        & \meanstd{94.05}{0.00}
        & \meanstd{6.70}{0.00}
        & \meanstd{4.44}{2.14}
        & \meanstd{1.85}{0.28} \\ 
        DiBS 
        & \meanstd{92.85}{0.00}
        & \bestmeanstd{100.00}{0.00}
        & \meanstd{94.05}{0.00}
        & \bestmeanstd{100.00}{0.00}
        & \meanstd{5.12}{0.38}
        & \meanstd{1.77}{0.08} \\ 
        FiP
        & \meanstd{92.85}{0.00}
        & \meanstd{12.15}{0.00}
        & \meanstd{94.05}{0.00}
        & \meanstd{6.70}{0.00}
        & \meanstd{6.83}{0.00}
        & \meanstd{2.07}{0.00} \\
        CAFE
        & \bestmeanstd{98.05}{0.07}
        & \bestmeanstd{100.00}{0.00}
        & \bestmeanstd{97.98}{0.06}
        & \bestmeanstd{100.00}{0.00}
        & \bestmeanstd{1.99}{0.01}
        & \bestmeanstd{0.72}{0.02} \\ 
        \bottomrule
    \end{tabular}
\end{table}
\vspace{-2pt}
\paragraph{Sangiovese.}
We further evaluate CAFE on counterfactual samples generated from a Bayesian network fitted to vineyard experiments in Tuscany~\citep{magrini2017conditional}. The task is to predict how an alternative harvest time would have changed the sugar concentration (\texttt{Brix}) and acidity (\texttt{pH}) of an individual grape batch. These batches differ in vegetative vigour, crop load, and leaf condition. For each batch, we toggle the binary \texttt{HarvestTime} indicator, setting $t=1-t_{\mathrm f}$ to switch between standard and late harvesting. We use $N=1000$ samples to form the observational dataset and evaluate counterfactual predictions using mean absolute percentage error (MAPE).
Table~\ref{tab:causalman_results} shows that our model achieves the lowest MAPE on both Sangiovese outcomes, reaching $1.99\%$ for Brix and $0.72\%$ for pH. These correspond to relative error reductions of approximately $21\%$ and $59\%$ over the strongest competing baseline for each setting, respectively. CAFE also exhibits low variation across partition seeds, which reflects its ability to predict counterfactuals directly without relying on a separately estimated causal graph or ordering.

Taken together, CAFE retains strong individual counterfactual prediction performance across manufacturing and viticulture settings, even when unobserved variables or non-additive noise place the data-generating process outside its training assumptions. This provides empirical support for the practical applicability of counterfactual prediction learned from a synthetic SCM prior.

%% file: data/appendix.tex
\section{Proof} 

\subsection{Proof of Proposition~\ref{prop:counterfactual_identifiability}} 
\label{app:counterfactual_identifiability} 
\begin{proposition}[Individual counterfactual identifiability]
\label{prop:counterfactual_identifiability_formal}
Let $\mathfrak M$ be the class of SCMs satisfying
Assumption~\ref{ass:additive_noise}, and let
$\mathcal M\in\mathfrak M$ have graph $G$ and induce
$P_{\mathbf V}$.
For each node $i$, let
\[
S_i
:=
\operatorname{supp}
\bigl(P_{\mathbf V_{\operatorname{pa}_G(i)}}\bigr)
\]
denote the joint observational support of its parent variables.
Fix an intervention variable $T=V_j$, an outcome $Y=V_k$
with $j\neq k$, and a factual state
$\mathbf v\in\operatorname{supp}(P_{\mathbf V})$.
Write $\mathbf v^t$ for the counterfactual state under
$\operatorname{do}(T=t)$ in $\mathcal M$.
Suppose that:
\begin{enumerate}
    \item[(i)]
    $G$ is identifiable from $P_{\mathbf V}$
    within $\mathfrak M$;

    \item[(ii)]
    every model in $\mathfrak M$ inducing $P_{\mathbf V}$
    satisfies $\mathbb E|V_i|<\infty$,
    $\mathbb E|U_i|<\infty$, and continuity of $f_i$
    on $S_i$ for every node $i$;

    \item[(iii)]
    $\mathbf v^t_{\operatorname{pa}_G(i)}\in S_i$
    for every $i\neq j$.
\end{enumerate}
Then $\mathbf v^t$ is uniquely determined by
$(P_{\mathbf V},\mathbf v,j,t)$ within $\mathfrak M$.
In particular, the individual counterfactual outcome
$y^t=v_k^t$ is identifiable.
\end{proposition}

\begin{proof}
We show that observationally equivalent additive mechanisms
differ only by constants on their observational supports,
and that these constants cancel when computing counterfactuals.
This cancellation is a form of invariance under invertible
reparameterizations of the exogenous
variables~\citep{nasr2023counterfactual}.

Let $\widetilde{\mathcal M}\in\mathfrak M$ be any model
inducing $P_{\mathbf V}$, with structural functions
$\widetilde f_i$ and exogenous variables $\widetilde U_i$.
By condition~(i), both models have the same graph $G$.
We show that their counterfactual states agree for the
given factual state and intervention.

In an acyclic SCM, the parents of node $i$ are functions
of exogenous variables other than $U_i$.
Mutual independence of the exogenous variables therefore gives
\[
U_i\perp\mathbf V_{\operatorname{pa}_G(i)}
\]
in $\mathcal M$, and the analogous independence holds
in $\widetilde{\mathcal M}$.
Define
\[
\mu_i=\mathbb E_{\mathcal M}[U_i],
\qquad
\widetilde\mu_i
=
\mathbb E_{\widetilde{\mathcal M}}[\widetilde U_i].
\]
By integrability and equality of the observational
distributions,
\begin{equation*}
f_i(\mathbf x)+\mu_i
=
\widetilde f_i(\mathbf x)+\widetilde\mu_i
\quad
P_{\mathbf V_{\operatorname{pa}_G(i)}}\text{-almost surely},
\end{equation*}
because both sides are versions of the same conditional mean.

Both sides are continuous on $S_i$ by condition~(ii).
If they differed at a point of $S_i$, continuity would imply
a difference on a relative neighborhood with positive
parent probability, contradicting their almost-sure equality.
Hence
\begin{equation}
\widetilde f_i(\mathbf x)
=
f_i(\mathbf x)+c_i,
\qquad
\mathbf x\in S_i,
\qquad
c_i:=\mu_i-\widetilde\mu_i.
\label{eq:mechanism_shift}
\end{equation}

Since $\mathbf v\in\operatorname{supp}(P_{\mathbf V})$,
its parent subvectors belong to the corresponding sets $S_i$.
The exogenous realizations recovered from $\mathbf v$
in the two models therefore satisfy
\begin{equation*}
\begin{aligned}
u_i
&=
v_i-f_i(\mathbf v_{\operatorname{pa}_G(i)}),\\
\widetilde u_i
&=
v_i-\widetilde f_i(\mathbf v_{\operatorname{pa}_G(i)})
=
u_i-c_i.
\end{aligned}
\end{equation*}
Let $\mathbf v^t$ and $\widetilde{\mathbf v}^{\,t}$
denote the counterfactual states obtained by keeping these
respective realizations fixed under $\operatorname{do}(T=t)$.

Proceed in a topological order of $G$.
At the intervened node, both models assign
$v_j^t=\widetilde v_j^t=t$.
At any other node $i$, suppose the two counterfactual states
agree on its parents; this condition is vacuous for a root.
Their common parent configuration equals
$\mathbf v^t_{\operatorname{pa}_G(i)}$, which belongs to
$S_i$ by condition~(iii).
Eq.~\ref{eq:mechanism_shift} therefore gives
\begin{equation*}
\begin{aligned}
\widetilde v_i^t
&=
\widetilde f_i
\left(\mathbf v^t_{\operatorname{pa}_G(i)}\right)
+\widetilde u_i\\
&=
f_i\left(\mathbf v^t_{\operatorname{pa}_G(i)}\right)
+c_i+u_i-c_i\\
&=
v_i^t.
\end{aligned}
\end{equation*}
Induction yields
$\widetilde{\mathbf v}^{\,t}=\mathbf v^t$.
Since $\widetilde{\mathcal M}$ was arbitrary,
all models in $\mathfrak M$ inducing $P_{\mathbf V}$
give the same counterfactual state and, in particular,
the same outcome $y^t=v_k^t$.
\end{proof}

To obtain Eq.~\ref{eq:counterfactual_population_posterior},
apply the proposition separately to each fixed-graph
subclass of SCMs supported by $\Pi$.
Within each subclass, condition~(i) is automatic.
Assuming conditions~(ii)--(iii) hold for the given
observational distribution and counterfactual query
in every subclass with positive graph posterior probability,
each graph $G$ determines a unique outcome $y_G^t$.
The conditional counterfactual posterior given $G$
is therefore a point mass at $y_G^t$.
Averaging these point masses with the graph posterior
probabilities gives the stated mixture.

\subsection{Population limit of the counterfactual posterior}

\label{app:counterfactual_posterior_limit}
\begin{proposition}[Population limit of the counterfactual posterior]
\label{prop:counterfactual_posterior_limit}
Let $\Pi$ be a prior over fully observed, acyclic ANMs with
mutually independent exogenous noises.
Suppose the SCM parameter spaces are standard Borel, the
observational and counterfactual maps are measurable, and
$\mathbf v$ and the observational samples are conditionally
i.i.d.\ given the SCM.
Fix an intervention value $t$.
Assume that, almost surely under the joint model induced by
$\Pi$, conditions (ii) and (iii) of the formal identifiability
proposition in Appendix~A.1 hold for $P_{\mathbf V}$ and
$(\mathbf v,t)$ within every fixed-graph subclass with
$p_\Pi(G\mid P_{\mathbf V})>0$.
Then
$$
    p^\star_\Pi(\,\cdot\mid\mathbf v,t,\mathcal D_{\mathrm{obs}})
    \;\Rightarrow\;
    p^\star_\Pi(\,\cdot\mid\mathbf v,t,P_{\mathbf V})
$$
as $n\to\infty$, almost surely under this joint model.
Here, $\Rightarrow$ denotes weak convergence.
Graph identifiability is not required.
\end{proposition}
\begin{proof}
All conditional laws below are understood as regular conditional
distributions under the joint model induced by $\Pi$, and all
expectations are taken under this model.
Consider an infinite observational sequence, with
$\mathcal D_{\mathrm{obs}}$ containing its first $n$ samples,
and keep the same individual as $n$ increases.
Define
$$
    \mathcal F_n
    =
    \sigma\bigl(
        \mathbf v,\mathbf v^{(1)},\ldots,\mathbf v^{(n)}
    \bigr),
    \qquad
    \mathcal F_\infty
    =
    \sigma\left(\bigcup_{n\geq1}\mathcal F_n\right).
$$
Under the ANM assumption,
$Y^t=\mathcal Y(\mathbf v,t;G,\theta)$.
Conditioning on the SCM therefore gives
$$
\begin{aligned}
    \mathcal L_\Pi(Y^t\mid\mathcal F_n)
    &=
    \sum_{G\in\mathcal G_\Pi}
    p_\Pi(G\mid\mathbf v,\mathcal D_{\mathrm{obs}})
    \int_{\Theta_G}
    \delta_{\mathcal Y(\mathbf v,t;G,\theta)}\,
    p_\Pi(d\theta\mid G,\mathbf v,\mathcal D_{\mathrm{obs}})
    \\
    &=
    p^\star_\Pi(\,\cdot\mid\mathbf v,t,\mathcal D_{\mathrm{obs}}),
\end{aligned}
$$
which is Eq.~\ref{eq:post_finite}.
The infinite observational sequence determines
$P_{\mathbf V}$ almost surely.
Indeed, the strong law of large numbers, applied to a countable
convergence-determining family of bounded continuous functions,
yields
$$
    \frac{1}{n}\sum_{r=1}^{n}\delta_{\mathbf v^{(r)}}
    \;\Rightarrow\;P_{\mathbf V}
    \qquad\text{almost surely}.
$$
Thus, $P_{\mathbf V}$ has an
$\mathcal F_\infty$-measurable version.

Let $\mathcal D_\infty$ denote the infinite observational
sequence.
By conditional i.i.d.\ sampling,
$$
    \mathcal L_\Pi(\mathcal D_\infty\mid G,\theta,\mathbf v)
    =
    P_{\mathbf V}^{\otimes\mathbb N}.
$$
This conditional law depends on the SCM only through
$P_{\mathbf V}$.
Since $Y^t$ is determined by $(G,\theta,\mathbf v)$,
$\mathcal D_\infty$ and $Y^t$ are conditionally independent
given $(\mathbf v,P_{\mathbf V})$.
Together with the measurability established above, this gives
$$
    \mathcal L_\Pi(Y^t\mid\mathcal F_\infty)
    =
    \mathcal L_\Pi(Y^t\mid\mathbf v,P_{\mathbf V}).
$$
Moreover,
$\mathcal L_\Pi(\mathbf v\mid G,P_{\mathbf V})=P_{\mathbf V}$,
so
$p_\Pi(G\mid\mathbf v,P_{\mathbf V})
=p_\Pi(G\mid P_{\mathbf V})$.
We now apply the identifiability argument in Appendix~A.1
separately within each fixed-graph subclass.
Within such a subclass, condition (i) is automatic, while
conditions (ii) and (iii) hold by assumption.
Thus, each compatible graph with positive posterior probability,
together with $P_{\mathbf V}$, uniquely determines the
counterfactual outcome for $(\mathbf v,t)$.
Hence,
$$
\begin{aligned}
    \mathcal L_\Pi(Y^t\mid\mathcal F_\infty)
    &=
    \sum_{i=1}^{K}
    p_\Pi(G_i\mid P_{\mathbf V})\,
    \delta_{\mathcal Y
    (\mathbf v,t;G_i,\theta_{G_i}(P_{\mathbf V}))}
    \\
    &=
    p^\star_\Pi(\,\cdot\mid\mathbf v,t,P_{\mathbf V}),
\end{aligned}
$$
which is Eq.~\ref{eq:counterfactual_population_posterior}.

Finally, for any bounded continuous function
$h:\mathbb R\to\mathbb R$, L'evy's upward theorem gives
$$
\begin{aligned}
    &\int h(y)\,
    p^\star_\Pi(dy\mid\mathbf v,t,\mathcal D_{\mathrm{obs}})
    =
    \mathbb E_\Pi[h(Y^t)\mid\mathcal F_n]
    \\
    &\qquad\longrightarrow
    \mathbb E_\Pi[h(Y^t)\mid\mathcal F_\infty]
    =
    \int h(y)\,
    p^\star_\Pi(dy\mid\mathbf v,t,P_{\mathbf V})
    \qquad\text{almost surely}.
\end{aligned}
$$
Applying this result simultaneously to a countable
convergence-determining family of bounded continuous functions
on $\mathbb R$ establishes the claimed almost-sure weak
convergence.
\end{proof}

\subsection{Proof of Proposition~\ref{prop:exact_cate_adjustment}}
\label{app:proof_cepo_adj}

\begin{cateadjustmentrestated}[Exact counterfactual recovery by CATE adjustment]
Consider a fully observed, acyclic SCM with mutually independent
additive noises. Let $\mathbf{X}$ contain all endogenous variables
other than $T$ and $Y$, and assume that the CEPOs are finite
and continuous in the treatment value.

If $\mathbf{X}$ contains no descendants of $T$ and
$Y$ is not a parent of $T$, then, for almost every factual
observation $(\mathbf{x},t_{\mathrm f},y)$ and any intervention
value $t$ in the conditional support of
$T\mid\mathbf{X}=\mathbf{x}$, the CATE
$\mu(\mathbf{x},t)-\mu(\mathbf{x},t_{\mathrm f})$
is identifiable from the observational distribution, and
the counterfactual outcome under $\mathrm{do}(T=t)$ satisfies
\[
    y^t
    =
    y+\mu(\mathbf{x},t)
      -\mu(\mathbf{x},t_{\mathrm f}).
\]
Consequently, CATE adjustment recovers the individual
counterfactual exactly whenever the CEPO estimates are exact
at both the factual and intervention treatment values.
\end{cateadjustmentrestated}

\begin{proof}
We first establish identification of the CATE.
Since $Y$ is not a parent of $T$, all parents of $T$ are
contained in $\mathbf{X}$, so its structural equation can
be written as
\[
    T=h(\mathbf{X})+U_T.
\]
Because $\mathbf{X}$ contains no descendants of $T$,
it depends only on exogenous noises other than $U_T$.
The same holds for $Y^a$ under $\mathrm{do}(T=a)$, which
replaces the structural equation for $T$.
Independence of the exogenous noises therefore implies
$U_T\perp(\mathbf{X},Y^a)$ and hence
$Y^a\perp T\mid\mathbf{X}$.
Together with consistency, this gives
$$
    \mu(\mathbf{x},a)
    =
    \mathbb{E}[Y\mid\mathbf{X}=\mathbf{x},T=a]
$$
for $P_{\mathbf{X},T}$-almost every $(\mathbf{x},a)$.
For $P_{\mathbf{X}}$-almost every $\mathbf{x}$, continuity
of $\mu(\mathbf{x},\cdot)$ uniquely determines its values
throughout the conditional support of
$T\mid\mathbf{X}=\mathbf{x}$.
Thus, the CEPOs at $t$ and $t_{\mathrm f}$, and their
difference, are identifiable from the observational distribution.

We next establish individual counterfactual recovery.
Since $\mathbf{X}$ contains all endogenous variables other
than $T$ and $Y$, we can write
$$
    Y=g(\mathbf{X},T)+U_Y,
$$
where $g$ depends only on the parents of $Y$.
Intervening on $T$ preserves both $\mathbf{X}$ and the
individual's exogenous noise, so
$$
    y=g(\mathbf{x},t_{\mathrm f})+u_Y,
    \qquad
    y^t=g(\mathbf{x},t)+u_Y.
$$
Taking conditional expectations in
$Y^a=g(\mathbf{X},a)+U_Y$ yields
$$
    \mu(\mathbf{x},a)
    =
    g(\mathbf{x},a)
    +\mathbb{E}[U_Y\mid\mathbf{X}=\mathbf{x}].
$$
The conditional noise mean cancels when comparing the two
treatment values:
$$
    \mu(\mathbf{x},t)-\mu(\mathbf{x},t_{\mathrm f})
    =
    g(\mathbf{x},t)-g(\mathbf{x},t_{\mathrm f})
    =
    y^t-y.
$$
Rearranging establishes
Eq.~\ref{eq:exact_residual_calibration}.
The final claim follows by substituting exact CEPO estimates
into Eq.~\ref{eq:factual_residual_calibration}.
\end{proof}
\subsection{Identifiability of 7 C-Suite settings} 

\label{app:iden_csuite}

We use the zero-based indices and executable generating equations of
C-Suite~\citep{geffner2022DECI}. Let $s(z)=\log(1+e^z)$,
$a=\sqrt{2/3}$, $b=1/\sqrt3$, and
$\alpha=\sqrt{1-6(1/\sqrt5-1/3)}$.
Within each setting, the primitive noises are mutually independent:
$Z_i\sim\mathcal N(0,1)$, $L_i\sim\operatorname{Laplace}(0,1)$,
and $W_i\sim t_3$. Transformed primitive noises are absorbed into
$U_i$. We consider scalar variables and the continuous-treatment versions.
The sets $S_i$ are population parent supports.

For the graph-identification arguments below, we additionally assume
faithfulness of the generating and competing models to their DAGs,
and restrict $\mathfrak M$ to independent ANMs with $C^3$ mechanisms.
Under faithfulness, competing graphs have the same skeleton and
unshielded colliders. We also use the graphical fact that two distinct
Markov-equivalent DAGs have an oppositely oriented edge $A\to B$,
$B\to A$, with
$\operatorname{pa}_{G}(B)\setminus\{A\}
=\operatorname{pa}_{G'}(A)\setminus\{B\}=S$;
conditioning on $S$ would therefore give additive-noise representations
in both directions~\citep[Proposition~29(ii)]{peters2014causal}.

For condition (ii), it suffices to establish finite first moments of
all observed variables: in any observationally equivalent independent
ANM, $X_i=\widetilde f_i(\widetilde{\mathbf{PA}}_i)+\widetilde U_i$
is an integrable sum of independent random variables, so
$\mathbb E|\widetilde U_i|<\infty$.
Indeed, conditioning at a parent value $p$ with finite conditional
absolute moment gives
$\mathbb E|\widetilde U_i|
\leq\mathbb E|\widetilde f_i(p)+\widetilde U_i|
+|\widetilde f_i(p)|<\infty$.
Continuity of every competing mechanism follows from the model-class
restriction above.

\paragraph{Collider lingauss.}
The equations and treatment are
\[
X_0=Z_0,\qquad X_2=Z_2,\qquad
X_1=(X_0+X_2+Z_1)/\sqrt3,\qquad T=X_0.
\]
\begin{proof}
\textit{(i)} Under the additional faithfulness assumption, the only
nontrivial independence, $X_0\perp X_2$, identifies
$X_0\to X_1\leftarrow X_2$ uniquely. Specifically,
$\rho_{01}=\rho_{12}=1/\sqrt3$,
$\rho_{02\mid1}=-1/2$, and
$\rho_{01\mid2}=\rho_{12\mid0}=1/\sqrt2$.
Hence its Markov equivalence class contains only this DAG.

\textit{(ii)} All variables are Gaussian with finite absolute moments.
Consequently every observationally equivalent ANM has integrable
errors, and its mechanisms are continuous by the definition of
$\mathfrak M$.

\textit{(iii)} The joint density is positive on $\mathbb R^3$, so
$S_1=\mathbb R^2$. Replacing $x_0$ by any finite $t$ preserves
parent-support membership; the other non-intervened node is a root.
Proposition~\ref{prop:counterfactual_identifiability} therefore applies.
\end{proof}

\paragraph{Fork lin nongauss.}
The equations and treatment are
\[
X_1=Z_1,\qquad
X_0=aX_1+b\{s(1.8Z_0)-1\},\qquad
X_2=aX_1+b\{s(1.8Z_2)-1\},\qquad T=X_0.
\]
\begin{proof}
\textit{(i)} Under faithfulness, the skeleton is
$X_0-X_1-X_2$ with no collider.
For either leaf $B\in\{X_0,X_2\}$, write $B=aX_1+N$,
where $N$ has positive variance and support $[-b,\infty)$.
The conditional support of $X_1\mid B=y$ is
$(-\infty,(y+b)/a]$.
A reverse ANM $X_1=g(B)+E$, $E\perp B$, would make these supports
translates of one fixed support, forcing $g(y)=y/a+c$.
Its residual would then equal $-N/a-c$, whose covariance with
$B$ is $-\operatorname{Var}(N)/a\ne0$.
Thus neither fork edge can reverse, identifying
$X_0\leftarrow X_1\to X_2$.

\textit{(ii)} Since $s(z)\leq\log2+|z|$, all variables have finite
second moments. The preliminary integrability argument gives finite
first moments of every competing error; all competing mechanisms
are continuous.

\textit{(iii)} The only nonempty parent supports are
$S_0=S_2=\mathbb R$. Intervening on the leaf $X_0$ changes no
parent of another node. Hence all non-intervened parent configurations
remain in support, and the proposition applies.
\end{proof}

\paragraph{Fork nonlin gauss.}
The equations and treatment are
\[
X_1=Z_1,\qquad
X_0=\sqrt6e^{-X_1^2}+\alpha Z_0,\qquad
X_2=s(1-X_1)-1.5+\sqrt{0.15}Z_2,\qquad T=X_0.
\]
\begin{proof}
\textit{(i)} Under faithfulness, it remains to orient the two edges
of the fork skeleton. Each forward bivariate model has a Gaussian
cause, independent nondegenerate Gaussian noise, and a smooth
nonlinear mechanism. The bivariate ANM identifiability result excludes
both reverse directions~\citep[Theorem~20 and Corollary~22]{peters2014causal}.
Thus the graph is $X_0\leftarrow X_1\to X_2$.

\textit{(ii)} The exponential term is bounded, softplus grows at most
linearly, and Gaussian noises have finite second moments.
Therefore all observed first moments and all competing error first
moments are finite. Continuity holds throughout $\mathfrak M$.

\textit{(iii)} The joint density is positive on $\mathbb R^3$.
In particular, $S_0=S_2=\mathbb R$, and intervening on the leaf
$X_0$ changes no parent of another node. The proposition applies.
\end{proof}

\paragraph{Large backdoor.}
The equations are
\[
X_0=s(1.8Z_0)-1,\qquad
X_1=1.5\{s(X_0+1)+s(0.5)-3\}+0.25Z_1,
\]
\[
X_2=s(X_0+1)+s(0.5+Z_2)-3,\qquad
X_i=s(X_{i-2}+1)+s(0.5+Z_i)-3\quad(i=3,4,5,6),
\]
\[
X_7=1.5s(X_5+1)-1+0.3Z_7,\qquad
X_8=2-s\!\left(1+\frac{-1.3X_6+X_7}{3}\right)+0.6L_8,
\]
with $T=X_7$.
\begin{proof}

\textit{(i)} Under faithfulness, the skeleton and the collider
$X_6\to X_8\leftarrow X_7$ are identified.
A reversal of either edge adjacent to $X_0$ would make that neighbor
its sole parent, since a second incoming edge would introduce an
unshielded collider at $X_0$.
For $X_0-X_1$, the support of $X_0\mid X_1=y$ is always
$[-1,\infty)$. A reverse ANM would therefore have a constant
mechanism, contradicting dependence.
For $X_0-X_2$, the conditional support is
\[
\operatorname{supp}(X_0\mid X_2=y)
=[-1,s^{-1}(y+3)-1],
\]
whose length varies with $y$; these intervals cannot be translates
of a fixed noise support. Thus $X_0\to X_1$ and $X_0\to X_2$.
Preserving the absence of further unshielded colliders then orients
both branches away from $X_0$, identifying the entire DAG.

\textit{(ii)} Every structural expression is a finite composition
of affine maps and softplus functions, with Gaussian or Laplace
primitive noise. The linear-growth bound for softplus gives finite
second moments by topological induction. Hence all required first
moments are finite in every competing ANM, and continuity holds.

\textit{(iii)} Under $\operatorname{do}(X_7=t)$, only the parents
of $X_8$ change among non-intervened nodes.
Given $\mathbf W=(X_0,\ldots,X_6)$, the variable $X_7$ is Gaussian
with variance $0.3^2$ and continuous conditional mean.
Consequently
$S_8=\operatorname{supp}(X_6)\times\mathbb R$, so
$(x_6,t)\in S_8$ for every finite $t$ and every factual observation
in support. All other required parent configurations are unchanged.
The proposition applies.
\end{proof}

\paragraph{Nonlin Simpson.}
The equations are
\[
X_0=Z_0,\qquad X_1=s(1-X_0)-1.5+0.15Z_1,
\]
\[
X_2=\tanh(2X_1)+1.5X_0-1+\tanh(Z_2),\qquad
X_3=5\tanh\!\left(\frac{X_2-4}{5}\right)+3+0.1L_3,
\]
with $T=X_1$ and edges $0\to1$, $0\to2$, $1\to2$, $2\to3$.
\begin{proof}

\textit{(i)} Under faithfulness, apply the conditional-edge reversal
criterion stated above. Each possible reversed edge is excluded:
$X_0\to X_1$ is identifiable by the nonlinear Gaussian bivariate
result~\citep[Theorem~20 and Corollary~22]{peters2014causal}.
For $X_0\to X_2$ conditional on $X_1$, the forward model is
$B=1.5A+c+N$ with $A$ of full support and $N=\tanh(Z_2)$.
Reverse conditional supports are bounded intervals of fixed length,
forcing a reverse mechanism $g(B)=B/1.5+c'$; the resulting residual
has covariance $-\operatorname{Var}(N)/1.5$ with $B$, a contradiction.

For $X_1\to X_2$ conditional on $X_0=x_0$, put
$r=x_2-(1.5x_0-1)$. The reverse conditional support is a lower
half-line when $-2<r<0$ and an upper half-line when $0<r<2$.
These supports cannot be translates of one noise support.
Finally, for $X_2\to X_3$, write
$h(x)=5\tanh((x-4)/5)+3$, whose range is $(-2,8)$.
The smooth positive density of $X_2$, combined with Laplace noise,
makes $p(X_2\mid X_3=y)$ have a cusp at $h^{-1}(y)$ for
$y\in(-2,8)$, whereas it is smooth for $y\notin[-2,8]$.
A location family cannot change between these two regularity types.
Thus no required reverse ANM exists, and the DAG is identified.

\textit{(ii)} Gaussian and Laplace noises have finite moments,
$\tanh$ is bounded, and softplus has at most linear growth.
Therefore every observed variable has finite first moment,
as does every competing additive error. All competing mechanisms
are continuous.

\textit{(iii)} Here $S_1=\mathbb R$, $S_2=\mathbb R^2$, and
$S_3=\mathbb R$: $(X_0,X_1)$ has an everywhere-positive density,
and $1.5X_0$ gives $X_2$ full marginal support.
Under $\operatorname{do}(X_1=t)$, both $(x_0,t)$ and the propagated
$x_2^t$ therefore lie in their required parent supports.
The proposition applies despite the bounded error at $X_2$.
\end{proof}

\paragraph{Symprod Simpson.}
The equations are
\[
X_0=Z_0,\qquad X_1=2\tanh(2X_0)+0.1W_1,
\]
\[
X_2=0.5X_0X_1+0.5L_2,\qquad
X_3=\tanh(1.5X_0)+0.3Z_3,
\]
with $T=X_1$ and edges $0\to1$, $0\to2$, $1\to2$, $0\to3$.
\begin{proof}

\textit{(i)} Under faithfulness, again exclude the conditional
reversals of the four edges.
The edge $X_0\to X_3$ is identifiable by the nonlinear Gaussian
bivariate result~\citep[Theorem~20 and Corollary~22]{peters2014causal}.
For $X_0\to X_1$, let $q$ be the density of $0.1W_1$.
Since $f(x)=2\tanh(2x)$ is bounded,
$q(y-f(x))/q(y)\to1$ uniformly in $x$ as $y\to\infty$.
Thus $X_0\mid X_1=y$ converges to $\mathcal N(0,1)$, including
its first two moments. If these conditional laws formed a location
family, every member would consequently be Gaussian with variance one.
But its density is proportional to
$e^{-x^2/2}q(y-f(x))$, whose second factor is bounded, positive,
and nonconstant in $x$; this cannot be a Gaussian density.
Hence the reverse ANM is impossible.

For either edge entering $X_2$, fix the other parent to any nonzero
value. The conditional model is $B=cA+N$, where $c\ne0$,
$N\sim\operatorname{Laplace}(0,0.5)$, and $A$ has a smooth
positive density and finite variance.
The density $p(A\mid B=y)$ has exactly one cusp, at $A=y/c$.
A reverse location family must therefore have $g(y)=y/c+c'$, so
its residual is $-N/c-c'$, with covariance
$-\operatorname{Var}(N)/c\ne0$ with $B$.
Thus neither conditional edge can reverse, and the DAG is identified.

\textit{(ii)} All primitive noises have finite second moments.
Moreover,
\[
X_0X_1=2X_0\tanh(2X_0)+0.1X_0W_1
\]
 has finite second moment by independence and boundedness of $\tanh$.
Therefore all observed first moments, and hence all competing error
first moments, are finite. Continuity holds throughout $\mathfrak M$.

\textit{(iii)} Every additive error has an everywhere-positive density
on $\mathbb R$, so the joint density is positive on $\mathbb R^4$.
All parent supports are full Euclidean spaces, and every finite
intervention satisfies condition (iii). The proposition applies.
\end{proof}

\paragraph{Weak arrows.}
The equations are
\[
X_0=s(1.8Z_0)-1,\qquad
X_1=0.75\{s(X_0+1)+s(0.5)-3\}+0.75Z_1,
\]
\[
X_2=s(1.5-X_0)+s(0.5+Z_2)-3,\qquad
X_3=s(1.5-X_1)+s(0.5+Z_3)-3,
\]
\[
X_i=s(X_{i-2}+1)+s(0.5+Z_i)-3\quad(i=4,5,6),\qquad
X_7=1.5s(X_5+1)-1+0.3Z_7,
\]
\[
X_8=s\!\left(1+0.1\sum_{i=0}^{5}X_i+0.5X_6+0.7X_7\right)-2+0.5L_8,
\]
with $T=X_7$.
\begin{proof}

\textit{(i)} Under faithfulness, every edge into $X_8$ is compelled:
for each $i<8$, some other parent $k<8$ is nonadjacent to $i$,
giving the unshielded collider $X_i\to X_8\leftarrow X_k$.
As in large backdoor, the fixed conditional support
$\operatorname{supp}(X_0\mid X_1=y)=[-1,\infty)$ excludes
$X_1\to X_0$.
For $X_0-X_2$, the reverse conditional support is
\[
[\ell(y),\infty),\qquad
\ell(y)=\max\{-1,\,1.5-s^{-1}(y+3)\}.
\]
For $y>s(2.5)-3$, this endpoint is constantly $-1$.
A reverse ANM would therefore require a constant mechanism on that
entire upper interval, and hence identical conditional laws there.
This is impossible: writing $h(x)=s(1.5-x)-3$ and $q$ for the
 density of $s(0.5+Z_2)$, the conditional density ratio at two distinct
$x,x'>-1$ is proportional to
$q(y-h(x))/q(y-h(x'))$.
Its log derivative tends to $h(x)-h(x')\ne0$ as $y\to\infty$,
since $(\log q)'(u)=-(u-0.5)+o(1)$.
Thus $X_0\to X_2$ also holds.
The absence of additional unshielded colliders then orients both
upstream branches away from $X_0$, identifying the full graph.

\textit{(ii)} The affine and softplus mechanisms have at most linear
growth. Topological induction gives finite second moments for every
variable. Hence all observed and competing-error first moments are
finite; all competing mechanisms are continuous.

\textit{(iii)} Let $\mathbf W=(X_0,\ldots,X_6)$.
The conditional distribution of $X_7$ given $\mathbf W$ is Gaussian
with positive variance and continuous mean, so
\[
S_8=\operatorname{supp}(\mathbf W)\times\mathbb R.
\]
Intervening on $X_7$ changes only the parent configuration of $X_8$
among non-intervened nodes, replacing $(\mathbf w,x_7)$ by
$(\mathbf w,t)\in S_8$. All other required parent configurations
are unchanged. The proposition applies.
\end{proof}

\section{Pretraining Details}
\subsection{Prior data generation} 
\label{app:prior_data_generation}

We sample the graph size $d$ by rounding a draw from
$\operatorname{LogUniform}(3,64)$ to the nearest integer.
We then sample the target average in-degree as
$\bar{k}\sim\operatorname{Gamma}(\alpha=3,\theta=0.4)$,
where $\alpha$ and $\theta$ denote the shape and scale, respectively. The target edge count is determined by $d\bar{k}$, and edges are sampled under a random node ordering as described in Section~\ref{sec:prior_generation}. 

For each root node, we sample its distribution family as
\begin{equation*}
F_{\mathrm{root}}\sim\operatorname{Categorical}
\bigl(\{\mathrm{Gaussian},\mathrm{Uniform},\mathrm{Zipf}\};
\{0.5,0.3,0.2\}\bigr),
\end{equation*}
where $\operatorname{Categorical}(S;\mathbf p)$ denotes a distribution over the choices in $S$ with corresponding probabilities $\mathbf p$. For each non-root node, the structural function is a randomly initialized MLP with one input per parent and a scalar output. Its number of linear layers $L$ and hidden width $H$ are sampled as
\begin{gather*}
L \sim \operatorname{Categorical}
\bigl(\{1,2,3\};\,\{0.3,0.4,0.3\}\bigr),\\
H \sim \operatorname{Categorical}
\bigl(\{8,12,16,20,32,64\};\,\{0.2,0.2,0.2,0.2,0.1,0.1\}\bigr).
\end{gather*}
All hidden layers share width $H$, and $L=1$ gives an affine mechanism. Activations are sampled independently for each hidden layer from tanh, sine, negation, identity, LeakyReLU, SELU, SiLU, ReLU, Softplus, and Hardtanh.
For each MLP, we sample a weight scale $s$, a bias standard deviation $\tau$, and a weight-masking probability $\rho$:
\begin{gather*}
s \sim \operatorname{LogUniform}(0.7,2.0),
\qquad
\tau \sim \operatorname{Uniform}(0,0.5),\\
\rho \sim 0.7\,\delta_0
      + 0.3\,\operatorname{Uniform}(0.05,0.5).
\end{gather*}
For a layer with $n_{\mathrm{in}}$ inputs, weights and biases are initialized independently as
\begin{gather*}
W_{ab} = A_{ab}Z_{ab},
\qquad
A_{ab} \sim \operatorname{Bernoulli}(q),\\
Z_{ab} \sim \mathcal N\!\left(0,\frac{c^2}{n_{\mathrm{in}}q}\right),
\qquad
b_a \sim \mathcal N(0,\tau^2).
\end{gather*}
We use $c=s$ for hidden layers and $c=1$ for the output layer. Weights in intermediate linear layers are independently masked with probability $\rho$ ($q=1-\rho$), while weights in the first and output linear layers are not masked ($q=1$).

For each non-root node, we sample the noise family as
\begin{equation*}
F_{\mathrm{noise}}\sim\operatorname{Categorical}
\bigl({\mathrm{Gauss},\mathrm{Lap},\mathrm{Gum}, \mathrm{Uni},
\mathrm{Exp},\mathrm{None}};
{0.5,0.1,0.1,0.1,0.1,0.1}\bigr). 
\end{equation*}
 The SNR $r$ is sampled independently for each mechanism through
\begin{equation*}
z\sim\operatorname{Uniform}(\log 0.5,\log 20),
\qquad
r=20.5-\exp(z).
\end{equation*}
This reflected log-uniform distribution spans $[0.5,20]$ and favors larger SNRs.
Each SCM generates 2000 factual--counterfactual pairs.
Before selecting the intervention value, we sample
$\eta \sim \mathrm{Uniform}(0.4,0.6)$ and randomly assign a fraction
$\eta$ of the individuals to the observational group and the remainder
to the target group.
The factual observations of the observational group form
$\mathcal D_{\mathrm{obs}}$.
We sample one intervention value uniformly from the factual treatment
values in this group and apply it to all individuals, sharing each
individual's exogenous noise between factual and counterfactual
evaluations.
\subsection{Training hyperparameters}   
\label{app:hyperparameters}
We train CAFE for $300$ epochs on 8 NVIDIA RTX 5090 GPUs, with $1000$ training steps per epoch. We use the Adam optimizer with a learning rate of $2\times10^{-4}$ and weight decay of $10^{-3}$. The first $100$ epochs are used for learning-rate warmup, and gradients are accumulated over eight steps before each optimizer update.

\section{Related Works}
\subsection{Counterfactual Prediction Baselines} 
\label{app:baselines} 
CAREFL and CausalNF perform counterfactual inference given a causal ordering or DAG, and both papers describe extensions for partially known causal graphs. For a fair comparison, we use the same PC procedure to estimate a partially directed graph from observational data before applying each method's extension. For both methods, we follow the model architectures and hyperparameter settings provided in their respective official repositories.\footnote{\url{https://github.com/piomonti/carefl}}\footnote{\url{https://github.com/psanch21/causal-flows}}
\vspace{-2pt}
\paragraph{CAREFL.}
Following the multivariate discovery procedure of CAREFL~\citep{khemakhem2021caurefl}, we use its likelihood-ratio criterion to orient the edges left undirected by PC. We then fit an affine autoregressive flow according to the resulting causal ordering and use it for counterfactual prediction.
\vspace{-2pt}
\paragraph{CausalNF.}
Following the partial-graph extension of CausalNF~\citep{javaloy2023causalNF}, we represent unresolved edges in both directions and group the resulting strongly connected components into blocks. Each block's joint distribution is modeled using an autoregressive flow with an arbitrary fixed ordering. The flow respects the causal ordering between blocks, with inter-block connections expanded to block-level dependencies.
\vspace{-2pt}
\paragraph{VCI.}
We follow the network architecture and training hyperparameters used in the original VCI~\citep{wu2025counterfactual} single-cell experiments, as in the official implementation.\footnote{\url{https://github.com/yulun-rayn/variational-causal-inference}} We use the designated treatment and outcome as $T$ and $Y$, respectively, and the remaining factual measurements as covariates $X$.  \vspace{-2pt}
\paragraph{SCIGAN.}
SCIGAN~\citep{bica2020estimating} is designed to estimate conditional mean treatment responses, 
its inference model does not condition on the individual's factual treatment or outcome. For individual counterfactual prediction, we omit this model and directly use the generator, conditioning on the individual's observations, including factual treatment and outcome, to predict under the queried intervention. We use the generator and dosage discriminator architectures and training hyperparameters provided in the official SCIGAN repository.\footnote{\url{https://github.com/ioanabica/SCIGAN}}
 
We include two Bayesian baselines, DECI and DiBS. DECI captures graph uncertainty through a variational posterior over causal graphs. We use the JointDiBS variant of DiBS, which approximates a joint posterior over graphs and mechanism parameters.   
\vspace{-2pt}
\paragraph{DECI.}
We implement DECI~\citep{geffner2022DECI} using the official Causica library.\footnote{\url{https://github.com/microsoft/causica}}
We jointly learn a variational distribution over causal graphs and nonlinear additive-noise mechanisms from observational data, using spline flows to model the noise distributions.
Following the official evaluation procedure, we sample 10 graphs from the learned variational posterior.
For each graph, we infer the individual's exogenous noise from their factual observations and reuse it to predict the outcome under the intervention.
We average the resulting counterfactual predictions across the sampled graphs to obtain the final point prediction. 
\vspace{-2pt}
\paragraph{DiBS.}
We implement DiBS~\citep{lorch2021dibs} using the official library.\footnote{\url{https://github.com/larslorch/dibs}}
We use JointDiBS to approximate the joint posterior over causal graphs and mechanism parameters from observational data. We model the causal mechanisms using neural networks with additive Gaussian noise. We retain the official defaults for the network architecture, noise variance, and optimizer settings. We standardize variables using the training data and run Stein variational gradient descent with 20 particles for 1000 steps.
\vspace{-2pt}
\paragraph{FiP.}
We implement FiP~\citep{scetbon2024fixed} using the official repository.\footnote{\url{https://github.com/microsoft/causica/tree/main/research_experiments/fip}} We first train a causal ordering model using the provided pretraining code. For each evaluation dataset, this pretrained model predicts a causal ordering from the observational training data. Given the predicted ordering, we fit FiP's Transformer-based SCM to the same training data, following the architecture and training hyperparameters in the repository. 
\vspace{-2pt}
\subsection{Relation to other causal foundation models}  
Pearl's ladder of causation distinguishes three levels: association ($\mathcal L_1$), intervention ($\mathcal L_2$), and counterfactual reasoning ($\mathcal L_3$)~\citep{pearl2018book}. These levels form a directional hierarchy: answering questions at a lower level does not generally suffice to answer questions at a higher one~\citep{bareinboim2022pearl}. 

Conventional tabular foundation models primarily address ($\mathcal L_1$), predicting outcomes from observed features. Most existing causal foundation models address ($\mathcal L_2$), estimating outcome distributions and treatment effects under interventions. CausalPFN~\citep{balazadeh2026causalpfn} learns the posterior distribution of conditional expected potential outcomes (CEPOs), $\mu_t(\mathbf{x})=\mathbb{E}[Y(t)\mid\mathbf{X}=\mathbf{x}]$, implemented for binary treatments under strong ignorability. CCPFN~\citep{stith2026causal} extends this target to continuous treatments, estimating treatment-response curves $t\mapsto\mu_t(\mathbf{x})$. CausalFM~\citep{ma2026foundation} learns predictive distributions of interventional outcomes or outcome contrasts with binary-treatment implementations for effect estimation under back-door, front-door, and instrumental-variable assumptions. Do-PFN~\citep{robertson2026pfn} targets conditional interventional distributions $p(y\mid\mathrm{do}(t),\mathbf{x})$ for binary treatments. MACE-TNP~\citep{dhir2026estimating} targets marginal interventional distributions $p(y\mid\mathrm{do}(t))$, averaging over uncertain graphs and mechanisms. \cite{reuter2026use} additionally condition interventional predictions on complete or partial graph and ancestral information. Even accurate predictions of these interventional quantities do not generally determine what would have happened to a particular individual given their factual observations.

CAFE addresses this individual counterfactual prediction task at $\mathcal L_3$. The work most closely related to CAFE is  TabPFN-CFM~\citep{zhu2026causalfoundationmodelstructure}, which also claims to support counterfactual prediction. However, TabPFN-CFM requires the counterfactual values of all other input variables to predict an individual's counterfactual outcome (as stated in Section 3.2 of the paper). For treatment descendants, these values are generally unavailable and reveal intervention-induced changes that would otherwise require inference, simplifying outcome prediction. CAFE instead uses only observational data, the individual's factual state, and the specified intervention. Furthermore, TabPFN-CFM does not establish conditions for individual counterfactual identifiability. Our analysis provides such conditions and characterizes the uncertainty arising when observationally equivalent graphs imply different counterfactual outcomes. We demonstrate that CAFE can capture this uncertainty by comparing its predicted counterfactual distributions with analytic reference posteriors. 
\vspace{-2pt}
\section{Additional Results}  
\subsection{Predictive Distributions under Identifiable Causal Structure}
\label{app:identifiable_distributions}
\begin{figure}[t]
    \centering
    \includegraphics[width=0.98\linewidth]{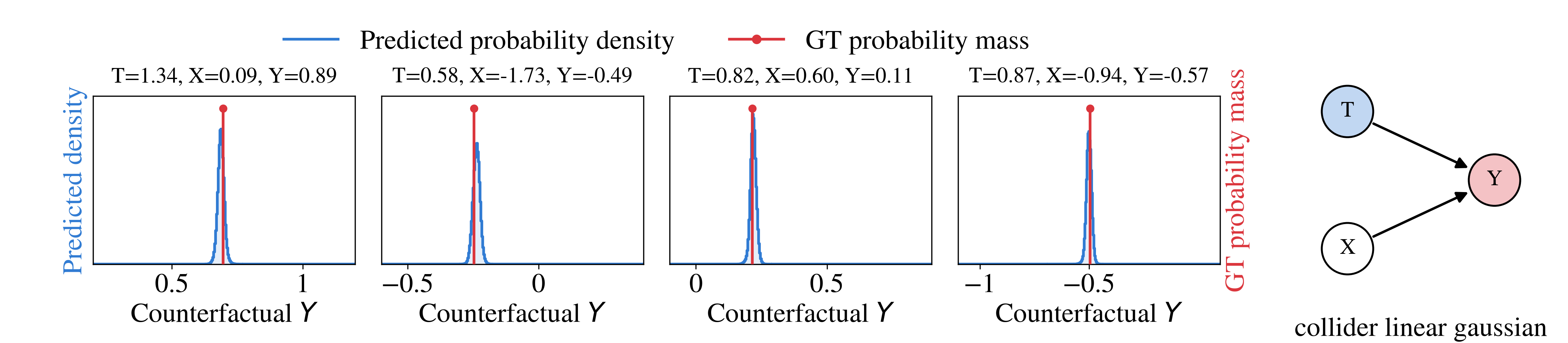}
    \par\vspace{0.0cm} 
    \includegraphics[width=0.98\linewidth]{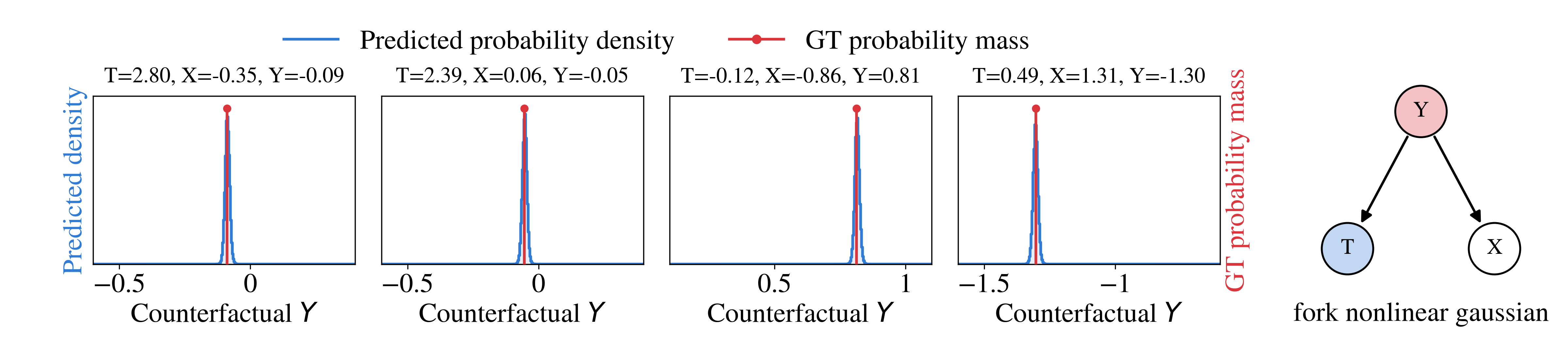}
    \caption{Counterfactual predictive distributions under identifiable causal structure. Predicted densities (blue) and ground-truth point masses (red) for four individuals from each of two C-Suite datasets: collider linear Gaussian (top) and fork nonlinear Gaussian (bottom). Panel titles report factual observations, and the corresponding causal graphs appear on the right. Each red marker represents unit probability mass at the individual's true counterfactual outcome; its height is read on a separate scale from the predicted density.}
    \label{fig:identifiable_distributions}
\end{figure} 
Figure~\ref{fig:identifiable_distributions} complements the point-prediction results in Section~\ref{sec:exp_csuite} by examining individual predictive distributions in two identifiable C-Suite settings. Under the conditions of Proposition~\ref{prop:counterfactual_identifiability}, the exact observational distribution, factual observations, and intervention uniquely determine each individual's counterfactual outcome, yielding a point-mass counterfactual distribution. Across the displayed individuals, our model produces sharply concentrated densities with peaks close to these ground-truth locations. In the collider setting, the predictions capture intervention-induced changes from the factual outcome. In the fork setting, where $Y$ is an ancestor of $T$, intervening on $T$ leaves $Y$ unchanged, and the predicted densities concentrate near the factual outcome. The small remaining spread and offsets are compatible with finite-context uncertainty and approximation in the learned, discretized predictive distribution. Together with the non-identifiable experiment, these examples provide qualitative evidence that the model can produce concentrated predictions when the counterfactual is identifiable while retaining uncertainty when observationally indistinguishable SCMs imply different outcomes.
\vspace{-2pt}
\subsection{Non additive noise settings in C-Suite}
\label{app:non_add_csuite}
\vspace{-2pt}
We further evaluate on three C-Suite datasets with non-additive discrete mechanisms: Large Backdoor Binary Treatment, Mixed Confounding, and Mixed Simpson. These settings contain binary or categorical nodes whose parents affect the probabilities of fixed outcome categories. Their structural equations therefore fall outside Assumption~\ref{ass:additive_noise}. Using the same evaluation protocol as in Section~\ref{sec:exp_csuite}, we compare counterfactual RMSE across all methods. Our method achieves the lowest RMSE on all three datasets and consistently improves upon the factual-outcome baseline, demonstrating robustness to these departures from the additive-noise assumption.  
\begin{table}[htbp]
    \centering
    \small
    \setlength{\tabcolsep}{5pt}
    \caption{Counterfactual prediction RMSE on C-Suite datasets
    with non-additive noise
    (mean $\pm$ standard deviation).
    Bold indicates the lowest mean RMSE in each column.}
    \label{tab:csuite_nonadditive}
    \begin{tabular}{lccc}
        \toprule
        Method
        & \makecell{large backdoor \\ binary treatment}
        & \makecell{mixed \\ confounding}
        & \makecell{mixed \\ simpson} \\
        \midrule
        y factual
        & \meanstd{0.17}{0.00}
        & \meanstd{0.86}{0.00}
        & \meanstd{0.50}{0.00} \\
        \midrule
        CAREFL
        & \meanstd{0.17}{0.00}
        & \meanstd{0.86}{0.00}
        & \meanstd{0.50}{0.00} \\
        CausalNF
        & \meanstd{0.23}{0.11}
        & \meanstd{0.86}{0.00}
        & \meanstd{0.64}{0.19} \\
        VCI
        & \meanstd{0.14}{0.01}
        & \meanstd{1.25}{0.09}
        & \meanstd{0.93}{0.34} \\
        SCIGAN
        & \meanstd{1.16}{0.04}
        & \meanstd{0.71}{0.07}
        & \meanstd{0.43}{0.06} \\
        DECI
        & \meanstd{0.14}{0.03}
        & \meanstd{0.73}{0.23}
        & \meanstd{1.13}{0.32} \\ 
        DiBS 
        & \meanstd{0.15}{0.01}
        & \meanstd{0.74}{0.08} 
        & \meanstd{0.56}{0.07} \\
        FiP
        & \meanstd{0.65}{0.08}
        & \meanstd{0.56}{0.00}
        & \meanstd{1.24}{0.00} \\
        CAFE
        & \bestmeanstd{0.05}{0.00}
        & \bestmeanstd{0.41}{0.00}
        & \bestmeanstd{0.38}{0.00} \\
        \bottomrule
    \end{tabular}
\end{table} 

\subsection{Comparison with Bayesian Structure-Learning Methods}
\label{app:bayesian_posterior_comparison}
We compare CAFE with DECI~\citep{geffner2022DECI} and
DiBS~\citep{lorch2021dibs} on the three-variable, six-DAG linear
Gaussian system in Section~4.2. DECI learns a variational posterior
over causal graphs, while DiBS jointly approximates a posterior
over graphs and mechanism parameters. This comparison examines
whether these methods accurately preserve counterfactual
uncertainty when the causal structure is not identifiable.

Table~\ref{app:posterior_prediction} shows that CAFE achieves the
lowest average error on all three metrics. DiBS improves on DECI
in mean RMSE and IQD, but has a larger variance RMSE, indicating
that more accurate posterior means do not necessarily imply more
accurate uncertainty estimates. The individual counterfactual predictions illustrate how the
baseline posteriors depart from the ground truth. DiBS assigns excessive
probability to the dominant regions of the ground truth posterior,
leaving too little mass on alternative counterfactual outcomes (Figure~\ref{app:posterior_comparison_deci_dibs}, top).
Thus, despite representing uncertainty over both graphs and
parameters, it does not faithfully preserve the relative
probabilities of plausible outcomes in these examples.
For DECI, the learned graph posterior concentrates on sparse
structures (Figure~\ref{app:posterior_comparison_deci_dibs}, bottom).
Its two most probable DAGs, $T\to X\to Y$ and $Y\to T\to X$,
receive $48.42\%$ and $45.02\%$ of the probability, respectively. These two-edge graphs receive $93.44\%$ of the probability, although neither
belongs to the six compatible complete DAGs. Most of the learned
graph probability therefore falls on structures that cannot
reproduce the population observational distribution. This
preference for sparse graphs is consistent with DECI's
sparsity-promoting graph prior, which may contribute to its
inaccurate counterfactual predictions in this fully connected
system. These results show that, under the evaluated configurations,
maintaining a distribution over causal graphs is not sufficient
to recover the ground truth counterfactual posterior. CAFE provides
more accurate posterior predictions and better preserves the
uncertainty arising from structural non-identifiability.
\begin{table}[htbp]
\centering
\small
\setlength{\tabcolsep}{4pt}
\caption{Counterfactual posterior prediction under non-identifiable
causal structure (mean $\pm$ standard deviation) at $N=1000$.}
\label{app:posterior_prediction}
\begin{tabular}{rccc}
\toprule
Method & Mean RMSE & Variance RMSE & IQD \\
\midrule
DECI & \meanstd{0.235}{0.045} & \meanstd{0.101}{0.030} & \meanstd{0.082}{0.050} \\
DiBS & \meanstd{0.142}{0.050} & \meanstd{0.175}{0.125} & \meanstd{0.034}{0.024} \\
CAFE & \bestmeanstd{0.094}{0.012} & \bestmeanstd{0.055}{0.010} & \bestmeanstd{0.018}{0.004} \\
\bottomrule
\end{tabular}
\end{table} 

\begin{figure}[t]
    \centering
    \includegraphics[width=0.95\linewidth]{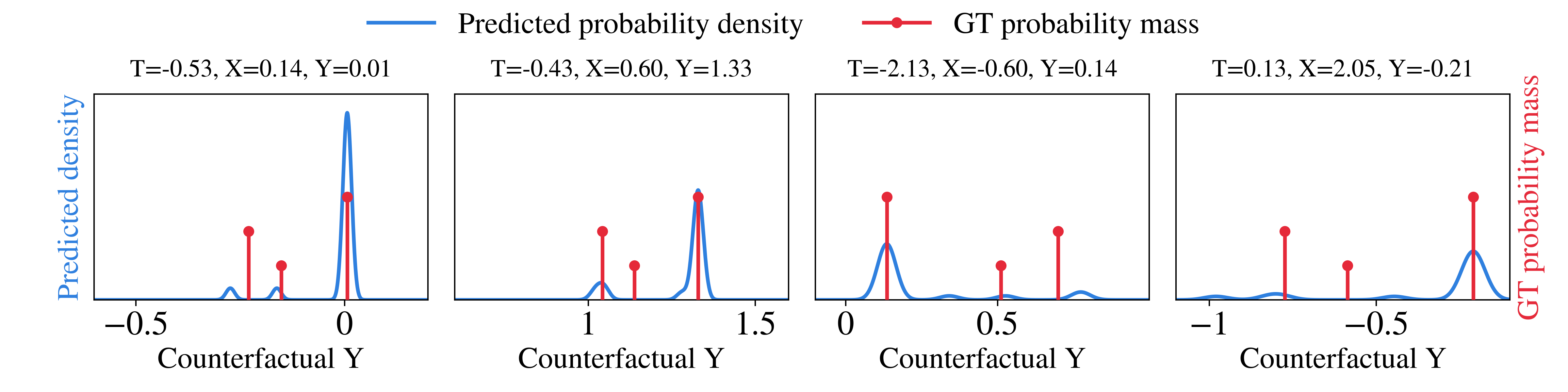}
    \par\vspace{0.0cm} 
    \includegraphics[width=0.95\linewidth]{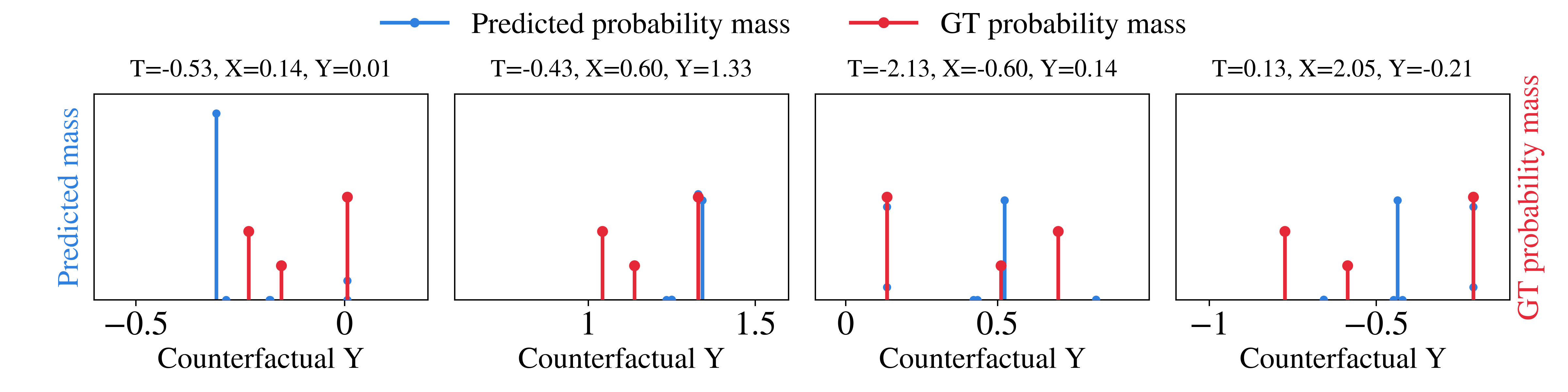}
    \caption{Estimated and ground-truth counterfactual posteriors in
the three-variable, six-DAG linear Gaussian system ($N=1000$).
Top: Counterfactual posterior estimates from DiBS.
Bottom: Counterfactual posterior estimates from DECI.}
    \label{app:posterior_comparison_deci_dibs}
\end{figure} 
\subsection{Running Time}
Table~\ref{tab:runtime_comparison} reports the time needed to apply each method to a new dataset, including graph estimation, model fitting, and counterfactual prediction where required. We exclude pretraining for CAFE and FiP because the pretrained models can be reused across downstream tasks.  

CAFE achieves the lowest runtime across all three datasets, as it requires neither graph search nor parameter updates on the evaluation data. We find that DiBS is particularly time-consuming. DiBS approximates the joint posterior over causal graphs and mechanism parameters by iteratively updating a set of particles. Each update requires Monte Carlo sampling of graphs and repeated likelihood and gradient computations. The resulting computation becomes more expensive as the number of variables increases, contributing to the long runtime observed on CausalMan Small.

\begin{table}[htbp]
    \centering
    \small
    \setlength{\tabcolsep}{4pt}
    \caption{Runtime comparison on Sangiovese and CausalMan
    (seconds, mean $\pm$ standard deviation; lower is better).
    The number of observed variables is shown for each dataset.}
    \label{tab:runtime_comparison}
    \begin{tabular}{lccc}
        \toprule
        Method
        & Sangiovese
        & CausalMan Micro
        & CausalMan Small \\
        & $d=15$
        & $d=24$
        & $d=53$ \\
        \midrule
        CAREFL
        & \meanstd{143.27}{15.45}
        & \meanstd{981.12}{62.54}
        & \meanstd{1064.83}{526.84} \\
        CausalNF  
        & \meanstd{21.13}{0.42} 
        & \meanstd{1453.74}{141.83} 
        & \meanstd{628.77}{57.50}\\
        VCI
        & \meanstd{276.73}{11.12}
        & \meanstd{211.85}{1.12}
        & \meanstd{208.36}{1.13} \\
        SCIGAN
        & \meanstd{233.44}{9.16}
        & \meanstd{117.91}{2.24}
        & \meanstd{122.23}{1.94} \\
        DECI
        & \meanstd{78.33}{1.04}
        & \meanstd{1065.45}{9.23}
        & \meanstd{214.36}{3.72} \\ 
        DiBS 
        & \meanstd{6638.62}{165.85}
        & \meanstd{9419.14}{983.09}
        & \meanstd{48251}{1354.46}\\
        FiP      
        & \meanstd{270.44}{5.36} 
        & \meanstd{84.71}{0.63} 
        & \meanstd{4405.92}{119.63} \\
        CAFE
        & \bestmeanstd{0.203}{0.003}
        & \bestmeanstd{0.068}{0.001}
        & \bestmeanstd{0.122}{0.002} \\
        \bottomrule
    \end{tabular}
\end{table}

\subsection{Can CEPO Methods Predict Individual Counterfactuals?}  
\label{app:cepo} 
We investigate whether methods used to estimate conditional expected
potential outcomes (CEPOs) at the interventional level ($\mathcal{L}_2$)
can be adapted to predict individual counterfactual outcomes ($\mathcal{L}_3$).  For an individual with factual observations
$(\mathbf{x},t_{\mathrm f},y)$, let $t$ denote the treatment value
specified by the intervention $\mathrm{do}(T=t)$.
For pretreatment covariates $\mathbf{X}$, the CEPO is
\begin{equation}
    \mu(\mathbf{x},t)
    =
    \mathbb{E}[Y^t\mid\mathbf{X}=\mathbf{x}].
    \label{eq:cepo_definition}
\end{equation}
A natural adaptation adds the estimated CATE between the intervention
and factual treatment values to the observed outcome~\citep{cate_adj1,cate_adj2}:
\begin{equation}
    \hat y^t_{\mathrm{cal}}
    =
    y+\hat\mu(\mathbf{x},t)
      -\hat\mu(\mathbf{x},t_{\mathrm f}).
    \label{eq:factual_residual_calibration}
\end{equation}
We refer to this procedure as \emph{CATE adjustment} and examine
the conditions under which it recovers the true individual
counterfactual.

\begin{proposition}[Exact counterfactual recovery by CATE adjustment]
\label{prop:exact_cate_adjustment}
Consider a fully observed, acyclic SCM with mutually independent
additive noises. Let $\mathbf{X}$ contain all endogenous variables
other than $T$ and $Y$, and assume that the CEPOs are finite
and continuous in the treatment value.

If $\mathbf{X}$ contains no descendants of $T$ and
$Y$ is not a parent of $T$, then, for almost every factual
observation $(\mathbf{x},t_{\mathrm f},y)$ and any intervention
value $t$ in the conditional support of
$T\mid\mathbf{X}=\mathbf{x}$, the CATE
$\mu(\mathbf{x},t)-\mu(\mathbf{x},t_{\mathrm f})$
is identifiable from the observational distribution, and
the counterfactual outcome under $\mathrm{do}(T=t)$ satisfies
\begin{equation}
    y^t
    =
    y+\mu(\mathbf{x},t)
      -\mu(\mathbf{x},t_{\mathrm f}).
    \label{eq:exact_residual_calibration}
\end{equation}
Consequently, CATE adjustment recovers the individual
counterfactual exactly whenever the CEPO estimates are exact
at both the factual and intervention treatment values.
\end{proposition}
The ability of CATE adjustment to recover individual counterfactuals
depends on the causal structure.
Proposition~\ref{prop:exact_cate_adjustment} establishes that an
accurate CATE estimate guarantees exact recovery when
(i) none of the covariates is a descendant of treatment; and
(ii) the outcome is not a parent of treatment. We provide the proof of Proposition~\ref{prop:exact_cate_adjustment} in Appendix~\ref{app:proof_cepo_adj}. 
\begin{figure}[t]
    \centering
    \includegraphics[width=0.98\linewidth]{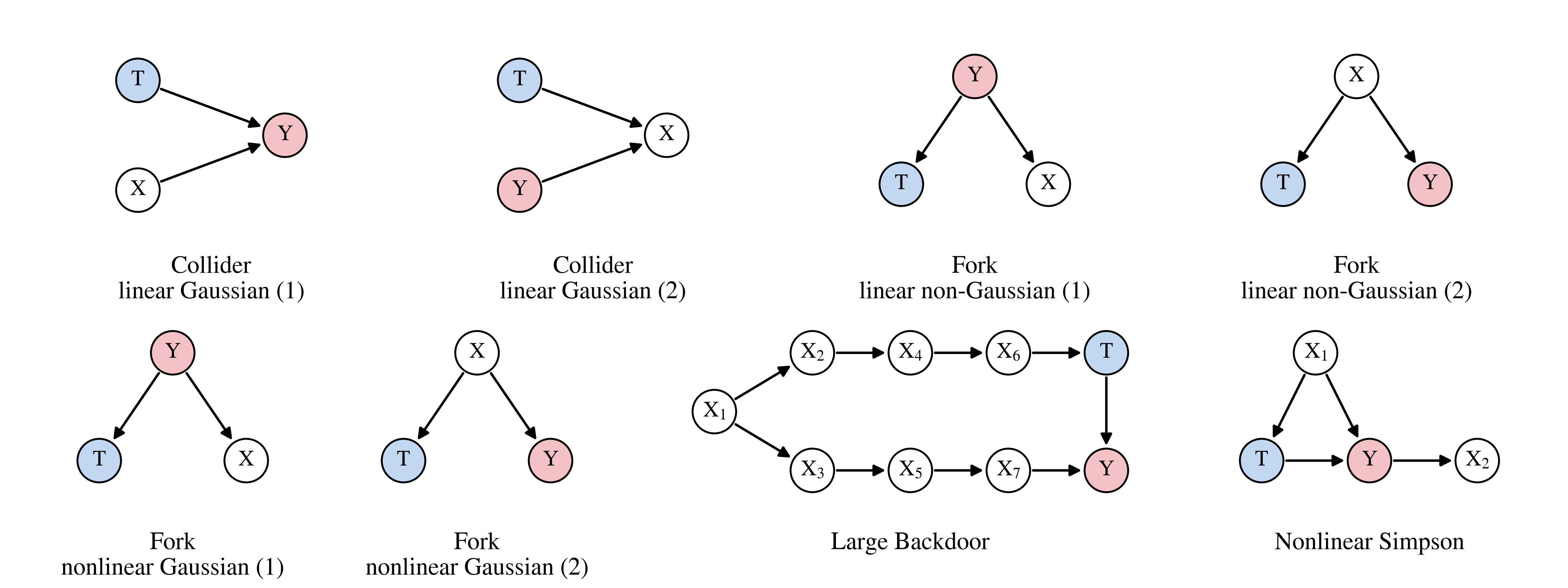}
    \caption{Eight causal graphs used to evaluate CATE adjustment.}
    \label{fig:cepo_dags}
\end{figure} 
\paragraph{Settings}In Section~\ref{sec:exp_csuite}, we pool prediction errors across targets within each graph when computing RMSE. Here, we report RMSE
separately for each target because different choices of outcome within the same graph can satisfy or violate the conditions for
exact CATE adjustment. This yields eight treatment--outcome settings across five C-Suite datasets, as illustrated in Figure~\ref{fig:cepo_dags}. With all remaining observed variables used as covariates,
four settings satisfy the conditions in
Proposition~\ref{prop:exact_cate_adjustment}:
Collider linear Gaussian (1), both Fork datasets with outcome (2),
and Large Backdoor.
The other four settings violate these conditions.
In Collider linear Gaussian (2), the covariate is a common child
of treatment and outcome, so adjustment conditions on a collider
that is also a treatment descendant.
In both Fork datasets with outcome (1), the outcome directly causes treatment.
In Nonlinear Simpson, the covariates include a descendant of the outcome and, consequently, of treatment. Thus, in these four settings, CEPO methods with CATE adjustment are not guaranteed to recover individual counterfactual outcomes exactly.

\paragraph{Baselines} We compare CausalForestDML~\citep{chernozhukov2018double}, DRNet~\citep{DRNet}, CCPFN~\citep{stith2026causal}, and S-learners based on
LimiX~\citep{zhang2025limix} and TabPFN v3~\citep{tabpfn3}.
CausalForestDML directly estimates CATE. We obtain counterfactual predictions by
adding the estimated CATE to the factual outcome. Other baseline methods estimate CEPO. DRNet is a neural network trained on
the observational data, whereas CCPFN is a pretrained model that uses observational data as context at inference time.
We also evaluate S-learners based on LimiX and TabPFN v3, using each model to predict the outcome from the covariates and treatment.
For each CEPO method, we report both the estimated CEPO
$\hat\mu(\mathbf{x},t)$ used directly as the counterfactual
prediction and the prediction obtained through CATE adjustment
in Eq.~\ref{eq:factual_residual_calibration}.
\begin{table}[t]
    \centering
    \small
    \setlength{\tabcolsep}{4pt}
    \renewcommand{\arraystretch}{1.12}
    \caption{Counterfactual prediction RMSE on C-Suite datasets
    (mean $\pm$ standard deviation).
    Bold indicates the lowest mean RMSE in each column,
    excluding the factual-outcome baseline. Numbers in parentheses denote outcome indices.
    The superscript ``cal'' denotes an adjustment variant.
    \colorbox{calibrationblue}{Light blue} marks settings satisfying the two structural
conditions in Proposition~\ref{prop:exact_cate_adjustment}.}
    \label{app:csuite_results}

    \begin{tabularx}{\linewidth}{@{}l
    >{\columncolor{calibrationblue}\centering\arraybackslash}X
    >{\centering\arraybackslash}X
    >{\centering\arraybackslash}X
    >{\columncolor{calibrationblue}\centering\arraybackslash}X@{}}
        \toprule
        \addlinespace[4pt]
        & \multicolumn{2}{c}{\makecell{Collider\\linear Gaussian}}
        & \multicolumn{2}{c}{\makecell{Fork\\linear non-Gaussian}} \\
        \cmidrule(lr){2-3} \cmidrule(lr){4-5}
        Method
        & \multicolumn{1}{c}{(1)}
        & \multicolumn{1}{c}{(2)}
        & \multicolumn{1}{c}{(1)}
        & \multicolumn{1}{c@{}}{(2)} \\
        \midrule
        y factual
        & \meanstd{0.81}{0.00}
        & \meanstd{0.00}{0.00}
        & \meanstd{0.00}{0.00}
        & \meanstd{0.00}{0.00} \\
        \midrule
        CausalForestDML
        & \meanstd{0.16}{0.01}
        & \meanstd{0.72}{0.02}
        & \meanstd{0.51}{0.01}
        & \meanstd{0.30}{0.03} \\
        DRNet
        & \meanstd{0.64}{0.04}
        & \meanstd{0.95}{0.03}
        & \meanstd{0.76}{0.09}
        & \meanstd{0.75}{0.07} \\
        DRNet\textsuperscript{cal}
        & \meanstd{0.13}{0.01}
        & \meanstd{0.56}{0.09}
        & \meanstd{0.65}{0.10}
        & \meanstd{0.46}{0.13} \\
        CCPFN
        & \meanstd{0.58}{0.00}
        & \meanstd{0.98}{0.00}
        & \meanstd{0.56}{0.00}
        & \meanstd{0.57}{0.00} \\
        CCPFN\textsuperscript{cal}
        & \meanstd{0.04}{0.01}
        & \meanstd{0.68}{0.00}
        & \meanstd{0.45}{0.00}
        & \meanstd{0.06}{0.02} \\
        S-Learner + LimiX
        & \meanstd{0.58}{0.00}
        & \meanstd{0.99}{0.00}
        & \meanstd{0.60}{0.00}
        & \meanstd{0.56}{0.00} \\
        S-Learner + LimiX\textsuperscript{cal}
        & \meanstd{0.05}{0.01}
        & \meanstd{0.68}{0.00}
        & \meanstd{0.49}{0.00}
        & \meanstd{0.06}{0.01} \\
        S-Learner + TabPFN v3
        & \meanstd{0.58}{0.00}
        & \meanstd{1.01}{0.00}
        & \meanstd{0.61}{0.00}
        & \meanstd{0.55}{0.00} \\
        S-Learner + TabPFN v3\textsuperscript{cal}
        & \meanstd{0.04}{0.01}
        & \meanstd{0.71}{0.00}
        & \meanstd{0.49}{0.00}
        & \bestmeanstd{0.01}{0.00} \\
        \midrule
        CAFE
        & \bestmeanstd{0.03}{0.00}
        & \bestmeanstd{0.06}{0.01}
        & \bestmeanstd{0.00}{0.00}
        & \bestmeanstd{0.01}{0.00} \\
        \bottomrule
    \end{tabularx}

    \par\vspace{0.8em}

    \begin{tabularx}{\linewidth}{@{}l
    >{\centering\arraybackslash}X
    >{\columncolor{calibrationblue}\centering\arraybackslash}X
    >{\columncolor{calibrationblue}\centering\arraybackslash}X
    >{\centering\arraybackslash}X@{}}
        \toprule
        \addlinespace[4pt]
        & \multicolumn{2}{c}{\makecell{Fork\\nonlinear Gaussian}}
        & \multicolumn{1}{c}{Large}
        & \multicolumn{1}{c@{}}{Nonlinear} \\
        \cmidrule(lr){2-3}
        Method
        & \multicolumn{1}{c}{(1)}
        & \multicolumn{1}{c}{(2)}
        & \multicolumn{1}{c}{backdoor}
        & \multicolumn{1}{c@{}}{Simpson} \\
        \midrule
        y factual
        & \meanstd{0.00}{0.00}
        & \meanstd{0.00}{0.00}
        & \meanstd{0.74}{0.00}
        & \meanstd{1.36}{0.00} \\
        \midrule
        CausalForestDML
        & \meanstd{0.38}{0.01}
        & \meanstd{0.14}{0.01}
        & \meanstd{0.58}{0.06}
        & \meanstd{1.21}{0.01} \\
        DRNet
        & \meanstd{0.56}{0.00}
        & \meanstd{0.42}{0.04}
        & \meanstd{0.97}{0.03}
        & \meanstd{3.77}{0.17} \\
        DRNet\textsuperscript{cal}
        & \meanstd{0.40}{0.00}
        & \meanstd{0.14}{0.05}
        & \meanstd{0.47}{0.06}
        & \meanstd{3.76}{0.17} \\
        CCPFN
        & \meanstd{0.54}{0.00}
        & \meanstd{0.39}{0.00}
        & \meanstd{0.93}{0.01}
        & \meanstd{1.29}{0.02} \\
        CCPFN\textsuperscript{cal}
        & \meanstd{0.34}{0.00}
        & \meanstd{0.02}{0.00}
        & \meanstd{0.42}{0.01}
        & \meanstd{1.26}{0.03} \\
        S-Learner + LimiX
        & \meanstd{0.54}{0.00}
        & \meanstd{0.39}{0.00}
        & \meanstd{0.98}{0.01}
        & \meanstd{1.50}{0.04} \\
        S-Learner + LimiX\textsuperscript{cal}
        & \meanstd{0.37}{0.00}
        & \meanstd{0.03}{0.01}
        & \meanstd{0.53}{0.02}
        & \meanstd{1.47}{0.04} \\
        S-Learner + TabPFN v3
        & \meanstd{0.55}{0.00}
        & \meanstd{0.39}{0.00}
        & \meanstd{0.92}{0.02}
        & \meanstd{1.24}{0.05} \\
        S-Learner + TabPFN v3\textsuperscript{cal}
        & \meanstd{0.37}{0.00}
        & \meanstd{0.02}{0.00}
        & \meanstd{0.41}{0.04}
        & \meanstd{1.19}{0.05} \\
        \midrule
        CAFE
        & \bestmeanstd{0.00}{0.00}
        & \bestmeanstd{0.01}{0.00}
        & \bestmeanstd{0.15}{0.02}
        & \bestmeanstd{0.93}{0.01} \\
        \bottomrule
    \end{tabularx}
\end{table} 
\paragraph{Results}
Table~\ref{app:csuite_results} shows that CATE adjustment substantially improves counterfactual prediction in all four settings satisfying Proposition~\ref{prop:exact_cate_adjustment}. For example, adjustment reduces CCPFN's RMSE from $0.58$ to $0.04$ on Collider linear Gaussian (1), from $0.57$ to $0.06$ on Fork linear non-Gaussian (2), and from $0.39$ to $0.02$ on Fork nonlinear Gaussian (2). Large Backdoor also shows substantial improvements across all CEPO methods. Nevertheless, none of the adjusted methods outperforms CAFE. The remaining prediction errors may reflect inaccuracies in the estimated CEPOs.

In the four settings that violate these conditions, adjustment yields more limited improvements and leaves a substantial gap to CAFE. Collider linear Gaussian (2) illustrates both the limitation of this adaptation and the source of its partial improvement. Its graph is $T\to X\leftarrow Y$, so intervening on $T$ leaves $Y$ unchanged and the true counterfactual is $y^t=y$. The CEPO is therefore $\mu(x,t)=\mathbb{E}[Y\mid X=x]$, which does not depend on $t$. However, conditioning on the collider $X$ induces an association between $T$ and $Y$, so the observational regression $m(x,t)=\mathbb{E}[Y\mid X=x,T=t]$ varies with $t$. In this linear Gaussian setting, write $m(x,t)=\alpha x+\beta t+c$, with $\beta\neq0$, and let $r=y-m(x,t_{\mathrm f})$ denote the factual prediction residual. The prediction errors before and after adjustment are then
\begin{gather*}
m(x,t)-y^t = \beta(t-t_{\mathrm f})-r,\\
y+m(x,t)-m(x,t_{\mathrm f})-y^t
= \beta(t-t_{\mathrm f}).
\end{gather*}
Adjustment restores the individual's factual residual, which explains why it can reduce prediction error. It nevertheless retains the spurious treatment effect induced by conditioning on the collider. Thus, accurate observational prediction alone cannot eliminate the remaining error in this setting.

\subsection{Ablation study on the training prior} 
\label{app:albation} 
\begin{table}[t]
    \centering
    \small
    \setlength{\tabcolsep}{4pt}
    \renewcommand{\arraystretch}{1.12}
    \caption{Counterfactual prediction RMSE on C-Suite datasets
    (mean $\pm$ standard deviation). Full Prior denotes CAFE trained with its original prior over diverse DAGs. Backdoor Prior denotes the ablation model trained only on back-door graphs.
    Bold indicates the lowest mean RMSE in each column,
    excluding the factual-outcome baseline. Numbers in parentheses denote outcome indices.
    \colorbox{calibrationblue}{Light blue}  marks settings satisfying the two structural
conditions in Proposition~\ref{prop:exact_cate_adjustment}.}
    \label{app:prior_ablation}

    \begin{tabularx}{\linewidth}{@{}l
    >{\columncolor{calibrationblue}\centering\arraybackslash}X
    >{\centering\arraybackslash}X
    >{\centering\arraybackslash}X
    >{\columncolor{calibrationblue}\centering\arraybackslash}X@{}}
        \toprule
        \addlinespace[4pt]
        & \multicolumn{2}{c}{\makecell{Collider\\linear Gaussian}}
        & \multicolumn{2}{c}{\makecell{Fork\\linear non-Gaussian}} \\
        \cmidrule(lr){2-3} \cmidrule(lr){4-5}
        Method
        & \multicolumn{1}{c}{(1)}
        & \multicolumn{1}{c}{(2)}
        & \multicolumn{1}{c}{(1)}
        & \multicolumn{1}{c@{}}{(2)} \\
        \midrule
        y factual
        & \meanstd{0.81}{0.00}
        & \meanstd{0.00}{0.00}
        & \meanstd{0.00}{0.00}
        & \meanstd{0.00}{0.00} \\
        \midrule
        Backdoor Prior
        & \bestmeanstd{0.02}{0.00}
        & \meanstd{0.39}{0.00}
        & \meanstd{1.20}{0.06}
        & \bestmeanstd{0.01}{0.00} \\
 
        Full Prior 
        & \meanstd{0.03}{0.00}
        & \bestmeanstd{0.06}{0.01}
        & \bestmeanstd{0.00}{0.00}
        & \bestmeanstd{0.01}{0.00} \\
        \bottomrule
    \end{tabularx}

    \par\vspace{0.8em}

    \begin{tabularx}{\linewidth}{@{}l
    >{\centering\arraybackslash}X
    >{\columncolor{calibrationblue}\centering\arraybackslash}X
    >{\columncolor{calibrationblue}\centering\arraybackslash}X
    >{\centering\arraybackslash}X@{}}
        \toprule
        \addlinespace[4pt]
        & \multicolumn{2}{c}{\makecell{Fork\\nonlinear Gaussian}}
        & \multicolumn{1}{c}{Large}
        & \multicolumn{1}{c@{}}{Nonlinear} \\
        \cmidrule(lr){2-3}
        Method
        & \multicolumn{1}{c}{(1)}
        & \multicolumn{1}{c}{(2)}
        & \multicolumn{1}{c}{backdoor}
        & \multicolumn{1}{c@{}}{Simpson} \\
        \midrule
        y factual
        & \meanstd{0.00}{0.00}
        & \meanstd{0.00}{0.00}
        & \meanstd{0.74}{0.00}
        & \meanstd{1.36}{0.00} \\
        \midrule
        Backdoor Prior
        & \meanstd{0.26}{0.01}
        & \bestmeanstd{0.01}{0.00}
        & \bestmeanstd{0.13}{0.01}
        & \meanstd{1.01}{0.02} \\
      
        Full Prior 
        & \bestmeanstd{0.00}{0.00}
        & \bestmeanstd{0.01}{0.00}
        & \meanstd{0.15}{0.02}
        & \bestmeanstd{0.93}{0.01} \\
        \bottomrule
    \end{tabularx}
\end{table}
We examine how restricting the causal structures used for pretraining affects counterfactual prediction. Specifically, we train an ablation model restricted to the standard back-door setting considered by CausalPFN~\citep{balazadeh2026causalpfn} and CCPFN~\citep{stith2026causal}. We retain the original DAG sampling procedure but change how the intervention and outcome variables are selected. For each sampled DAG, we choose $T$ uniformly from nodes with exactly one descendant and designate that descendant as $Y$. If no such node exists, we resample the DAG. This yields a back-door structure: $T$ directly affects $Y$, while all remaining variables $\mathbf{X}$ are pretreatment covariates that may influence treatment assignment and the outcome. The sampling procedures for structural functions and exogenous noise, the model architecture, and the training hyperparameters remain unchanged.

\paragraph{Results.}
We find that the ablation model trained with back-door prior (Table~\ref{app:prior_ablation}) shows a similar pattern of results to the CATE-adjusted baselines (Table~\ref{app:csuite_results}). Under the back-door prior, the factual treatment $t_\text{f}$ and outcome $y$ provide information about the individual's outcome (centered) noise, i.e., $y-\mu(\mathbf{x},t_{\mathrm f})$ in Eq.\ref{eq:exact_residual_calibration}, which is shared between the factual and counterfactual outcomes. The ablation model can learn to infer this noise term from its factual inputs $t_{\mathrm f}$ and $y$, implicitly performing abduction, while the CATE adjustment restores it explicitly via $y-\hat{\mu}(\mathbf{x},t_{\mathrm f})$ in Eq.\ref{eq:factual_residual_calibration}. Both approaches can therefore recover the individual noise missing from a CEPO prediction and use the same noise under intervention to predict the individual's counterfactual outcome, explaining their strong performance in the four settings satisfying the proposition's conditions.   

The other four settings involve covariates affected by treatment ($T\rightarrow X$) or outcomes that cause treatment ($Y\rightarrow T$), which the back-door prior excludes. In these cases, the adjustment above is no longer guaranteed to recover the true counterfactual. The ablation model's RMSE increases relative to full CAFE in all four settings, including from $0.06$ to $0.39$ on Collider linear Gaussian (2) and from $0.00$ to $1.20$ on Fork linear non-Gaussian (1). These results show that providing factual observations and directly supervising counterfactual predictions does not overcome the limitations of a restricted training prior. The full model's advantage, with the same architecture and prediction objective, highlights the importance of covering diverse causal relationships during pretraining.